%% file: iclr2025_conference.tex
\documentclass{article} %
\usepackage{iclr2025_conference,times}

\input{math_commands.tex}

\usepackage[colorlinks,citecolor=blue,bookmarks=true]{hyperref}
\usepackage{url}
\usepackage[capitalize,noabbrev]{cleveref}
\usepackage{graphicx}
\usepackage{bm}
\usepackage{bbm}
\usepackage{tcolorbox}
\usepackage{subfigure}
\usepackage{booktabs} %
\input{glodef}

\input{locdef}

\def\withnotes{0}
\def\withcolors{0}
\def\withcolorsnew{0}

\def\ifarxiv{0}
\def\withcolorrevision{0}

\ifnum\withcolorrevision=1
\newcommand{\revise}[1]{{\textcolor{red}{#1}}}
\newcommand{\confirm}[1]{{\textcolor{black}{#1}}}
\else
\newcommand{\revise}[1]{{{#1}}}
\newcommand{\confirm}[1]{{\textcolor{black}{#1}}}
\fi

\ifnum\withcolorsnew=1

\newcommand{\znew}[1]{{\textcolor{brown}{#1}}}
\else

\newcommand{\znew}[1]{{\textcolor{black}{#1}}}
\fi

\ifnum\withcolors=1

\else

\fi

\ifnum\withnotes=0
\newcommand{\zs}[1]{{\noindent \textit{\small\textcolor{brown}{ziteng: #1}}}}
\newcommand{\ts}[1]{{\noindent \textit{\small\textcolor{brown}{theertha: #1}}}}

\else
\newcommand{\zs}[1]{{}}
\newcommand{\ts}[1]{{}}
\renewcommand{\ts}[1]{{}}
\fi

\newcommand{\msmall}{\cM_s}
\newcommand{\mbig}{\cM_b}
\newcommand{\nt}{\gamma}

\newcommand{\spec}{\textsc{SpecDec}}
\newcommand{\draftverify}{\textsc{Verify}}
\newcommand{\specblock}{\textsc{BlockVerify}\xspace}
\newcommand{\spectoken}{\textsc{TokenVerify}\xspace}
\newcommand{\specblockt}{block verification\xspace}
\newcommand{\spectokent}{token verification\xspace}
\newcommand{\acclength}{\beta}

\newcommand{\context}{{\bm{c}}}
\renewcommand{\a}{{\textsc{a}}}
\renewcommand{\b}{{\textsc{b}}}
\newcommand{\ber}{{\rm Ber}}
\newcommand{\diff}[1]{{\colorbox{blue!30}{#1}}}
\newcommand{\pab}{\p}

\newcommand{\tb}{h^{\rm block}}
\renewcommand{\tt}{h^{\rm token}}
\newcommand{\simp}{\sim_{\rm p}}

\title{Block Verification Accelerates Speculative Decoding}

\author{Ziteng Sun\thanks{All emails \texttt{@google.com}.} \\
Google Research\\
\texttt{zitengsun@} \\
\And
Uri Mendlovic\footnotemark[1]  \\
Google Research\\
\texttt{urimend@} \\
\And
Yaniv Leviathan\footnotemark[1]  \\
Google Research\\
\texttt{leviathan@} \\
\AND
Asaf Aharoni\footnotemark[1]  \\
Google Research\\
\texttt{asafaharoni@} \\
\And
Jae Hun Ro\footnotemark[1]  \\
Google Research\\
\texttt{jaero@} \\
\And
Ahmad Beirami\footnotemark[1]  \\
Google Research\\
\texttt{beirami@} \\
\And
Ananda Theertha Suresh\footnotemark[1]  \\
Google Research\\
\texttt{theertha@} \\
}

\iclrfinalcopy
\begin{document}

\maketitle
\begin{abstract}
Speculative decoding is an  effective method for lossless acceleration of large language models during inference. It uses a fast model to draft a block of tokens which are then verified in parallel by the target model, and provides a guarantee that the output is distributed identically to a sample from the target model. In prior works, draft verification is performed independently token-by-token. Surprisingly, we show that this approach is not optimal. We propose \emph{Block Verification}, a simple draft verification algorithm that verifies the entire block jointly and provides additional wall-clock speedup. We prove that the proposed mechanism is optimal in the expected number of tokens produced each iteration and specifically is never worse than the standard token-level verification.
Empirically, block verification provides modest but consistent wall-clock speedups over the standard token verification algorithm of 5\%-8\% in a range of tasks and datasets.
Given that block verification does not increase code complexity, maintains the strong lossless guarantee of the standard speculative decoding verification algorithm, cannot deteriorate performance, and, in fact, consistently improves it, it can be used as a good default in speculative decoding implementations.
\end{abstract}

\input{sections/intro}
\input{sections/bernoulli}
\input{sections/algorithm}

\input{sections/theory_main}
\input{sections/experiments}
\input{sections/related}

\section{Discussion}
We showed that the standard \emph{token verification} algorithm used by speculative decoding is not optimal.
Further, we proposed a new verification algorithm, \emph{block verification} and proved that it is an optimal verification algorithm.
We also demonstrated empirically that block verification consistently outperforms token
verification in a range of tasks.
While the theoretical proofs are somewhat involved, the actual implementation of block verification is not more complex than the standard algorithm (see \cref{sec:python_code}), and since our proposed algorithm can only perform better, never worse, than the standard token verification algorithm (see \cref{thm:optimal_decoded_length}), it can be used as a good default in speculative decoding implementations.

\bibliographystyle{abbrvnat}
\bibliography{references}

\newpage
\appendix
\input{sections/implementation}
\input{sections/theory}
\input{sections/greedy}
\input{sections/app_experiment}

\end{document}

%% file: math_commands.tex
\usepackage{amsmath,amsfonts,bm}

\def\eqref#1{equation~\ref{#1}}

\def\1{\bm{1}}

\DeclareMathAlphabet{\mathsfit}{\encodingdefault}{\sfdefault}{m}{sl}
\SetMathAlphabet{\mathsfit}{bold}{\encodingdefault}{\sfdefault}{bx}{n}

\DeclareMathOperator*{\argmax}{arg\,max}
\DeclareMathOperator*{\argmin}{arg\,min}

%% file: glodef.tex
\usepackage{xspace}
\usepackage{amsfonts,amsmath,amssymb,amsthm, pbox}

\usepackage{multirow}
\usepackage{dsfont} %
\usepackage{algorithmic,algorithm}

\usepackage[shortlabels]{enumitem}
\setitemize{noitemsep,topsep=0pt,parsep=0pt,partopsep=0pt}
\setenumerate{noitemsep,topsep=0pt,parsep=0pt,partopsep=0pt}

\usepackage[frozencache,cachedir=minted-cache]{minted}

\makeatletter
\@ifundefined{theorem}{%
  \theoremstyle{definition}
  \newtheorem{definition}{Definition}
  \theoremstyle{plain}
  \newtheorem{theorem}{Theorem}

  \newtheorem{lemma}{Lemma}

  \theoremstyle{remark}

}{}
\makeatother
\newenvironment{proofof}[1]{\noindent{\bf Proof of {#1}:~~}}{\(\qed\)}

\newcommand{\ignore}[1]{}

\newcommand{\EE}{\mathbb{E}}

\newcommand{\RR}{\mathbb{R}}

\def \cM     {{\cal M}}

\def \cX     {{\cal X}}

\newcommand{\eg}{\textit{e.g.,}\xspace}
\newcommand{\ie}{\textit{i.e.,}\xspace}  %

\def \Paren#1{{\left({#1}\right)}}

\newcommand{\eqdef}{{:=}}

\newcommand{\probof}[1]{\Pr\Paren{#1}}

\newcommand{\probofsub}[2]{\Pr_{#1}\Paren{#2}}

\def\ignore#1{}

\newcommand{\bi}{\begin{itemize}}
\newcommand{\ei}{\end{itemize}}

\def\orpro{\mathop{\mathchoice
   {\vee\kern-.49em\raise.7ex\hbox{$\cdot$}\kern.4em}
   {\vee\kern-.45em\raise.63ex\hbox{$\cdot$}\kern.2em}
   {\vee\kern-.4em\raise.3ex\hbox{$\cdot$}\kern.1em}
   {\vee\kern-.35em\raise2.2ex\hbox{$\cdot$}\kern.1em}}\limits}

\def\andpro{\mathop{\mathchoice
 {\wedge\kern-.46em\lower.69ex\hbox{$\cdot$}\kern.3em}
 {\wedge\kern-.46em\lower.58ex\hbox{$\cdot$}\kern.25em}
 {\wedge\kern-.38em\lower.5ex\hbox{$\cdot$}\kern.1em}
 {\wedge\kern-.3em\lower.5ex\hbox{$\cdot$}\kern.1em}}\limits}

\def\simge{\mathrel{%
   \rlap{\raise 0.511ex \hbox{$>$}}{\lower 0.511ex \hbox{$\sim$}}}}

\def\simle{\mathrel{
   \rlap{\raise 0.511ex \hbox{$<$}}{\lower 0.511ex \hbox{$\sim$}}}}

%% file: locdef.tex
\newcommand{\coloredcomment}[1]{\COMMENT{\textcolor{blue}{#1}}}

\newcommand{\pres}{\p_{\rm res}}
\newcommand{\prej}{\p_{\rm rej}}
\newcommand{\premain}{\p_{\rm remain}}

\newcommand{\eos}{{\rm E.O.S}}

\newcommand{\dtv}[2]{d_{\rm TV}(#1, #2)}

\newcommand{\p}{p}

%% file: sections/intro.tex
\section{Introduction}
Large language models (LLMs) \citep{chowdhery2022palm, touvron2023llama, openai2023gpt4, team2023gemini} are often decoded through autoregressive sampling, where 
generating $k$ tokens requires $k$ costly serial evaluations of the model.
To improve generation latency, \citet{leviathan2022fast} proposed \emph{speculative decoding}, which enables an LLM to generate several tokens concurrently.
In each iteration, conditioned on the current decoded prefix, a guess of the next block of $\nt$ tokens is made by a fast drafter (\eg a small model or a heuristic).
Each of the resulting $\nt+1$ prefixes are then evaluated by the large target model in parallel. To guarantee that the final output follows the same distribution as that of the large model, some of the generated tokens are accepted while others are rejected.
The accepted tokens\footnote{With an extra token sampled from either a residual distribution or the large model distribution.} are then appended to the prefix, and the process repeats until generation ends.
See \cref{fig:speculative-decoding-diagram} and \cref{alg:speculative_decoding_framework}.

\begin{figure*}[h]
\centering
\includegraphics[width=0.7\textwidth]{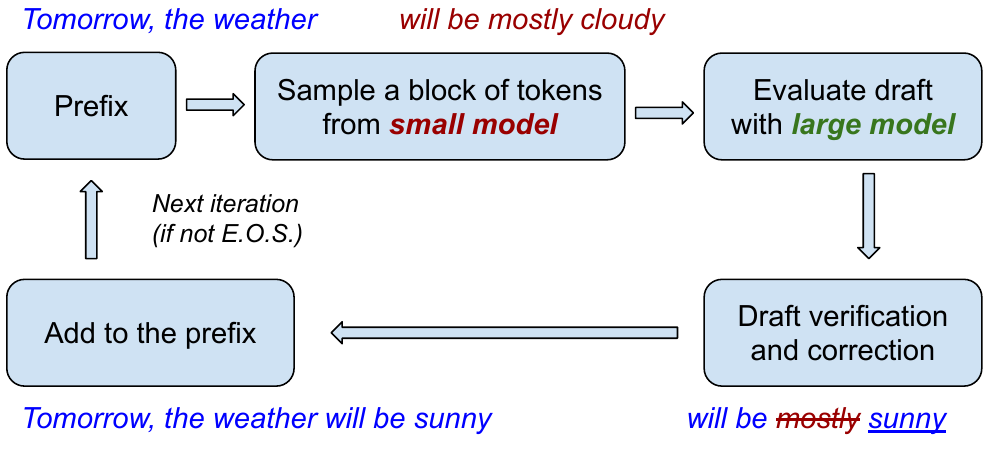}
\caption{One iteration of speculative decoding (\cref{alg:speculative_decoding_framework}). The prefixes and verified tokens are in blue, the unverified tokens from the draft model are in red, and the tokens sampled from the residual distribution are underlined.
}
\label{fig:speculative-decoding-diagram}
\end{figure*}

In \cite{leviathan2022fast}, the drafts are verified through a sequence of token-level rejection steps.
More specifically, given a prefix $\context$, let $X_1, X_2, \ldots, X_\nt$ be one sample block of length $\nt$ from the draft model $\msmall$, where $\forall i \le \nt$,
$
    X_i \sim \msmall(\cdot \mid \context, X^{i-1}).
$
Using the conditional distributions under the target large model $\mbig$ returned by the parallel evaluation step (\znew{$\forall 0 \le i \le \nt$,
$
    \mbig(\cdot \mid \context, X^{i})
$})
, the algorithm iterates over the draft tokens sequentially, and accepts each token $X_i$ with probability 
\begin{equation} \label{eqn:accept_probability}
     \min\left\{1, \frac{\mbig(X_i \mid \context, X^{i-1})}{\msmall(X_i \mid \context, X^{i-1})} \right\},
\end{equation}

The process continues until a token is rejected, at which point an extra token is sampled, for free, according to a residual distribution (see \cref{alg:token_verify} and \cite{leviathan2022fast} for more details). We refer to this algorithm as \emph{Token Verification}.
Since its introduction in \cite{leviathan2022fast}, this
token-by-token verification procedure has been the standard for follow-up works (see \cref{sec:related}).

In this work, we make the surprising observation that {\emph{the standard token verification algorithm, is not optimal}}, and propose a \emph{strictly better method}. Our key observation is that we can increase the number of accepted tokens, while maintaining the identical distribution guarantee, by \emph{jointly verifying the entire block of draft tokens} instead of verifying each token independently.  %
Our proposed algorithm, which we call \emph{Block Verification}, has the following advantages:
\begin{itemize}
  \item \textbf{Simple to use.} The algorithm is a plug-and-play replacement of the standard token verification algorithm of speculative decoding. It
  does not incur additional computation or code complexity costs. \znew{See \cref{alg:token_verify,alg:block_verify} for a side-by-side comparison.}
    \item \textbf{Identical distribution.} Importantly, our method is not an approximation and maintains the identical distribution guarantee of speculative decoding (\znew{\cref{thm:distribution_match}}).
    \item \textbf{\znew{Optimal} improvement.} With the same drafting model, the speedup of block verification is no worse than that of standard token verification.
    Moreover, we show that block verification is an optimal verification procedure (\cref{thm:optimal_decoded_length}).
\end{itemize}

We empirically test \emph{block verification} and compare it with the standard \emph{token verification} on a range of tasks and datasets. We show that our algorithm consistently improves over block efficiency (i.e. the expected number of generated tokens) by 7\%-10\% and overall empirical wall clock times by 5\%-8\% (see \cref{tab:exp_comparison_l_8}). Notably, our algorithm provides improvements only through  the  verification  phase  of  speculative decoding, and hence the improvements can be combined with improvements obtained from other works that aim at improving the drafting phase.
Since these improvements come for free, our block verification algorithm can be used as the draft verification algorithm by default in speculative decoding implementations.

%% file: sections/bernoulli.tex
\section{A Motivating Example} \label{sec:example}
The standard token verification algorithm  stochastically rejects draft tokens with a higher probability from $\msmall$ than from $\mbig$. This is necessary to guarantee that the generated tokens follow the same distribution as that of $\mbig$. Our main observation is that considering whether to reject a block of draft tokens jointly, instead of one-by-one, can result in accepting more tokens. We now illustrate this through a simple example.

Consider the following trivial language model whose token space consists only of 2 tokens: $\a$ and $\b$. Further, assume that both the large model $\mbig$ and the small model $\msmall$ are context-independent, and specifically that $\forall \context$,

\begin{equation}\label{eqn:toy_models}
    \mbig(\a) = 1/3, \quad
    \mbig(\b) = 2/3, \quad \msmall(\a) = 2/3, \quad \msmall(\b) = 1/3.
\end{equation}

In this setting, token verification will accept each draft token $X$ independently with probability 1 if $X = \b$ and $1/2$ if $X = \a$.
With a block size of $\gamma = 2$, since the total variation (TV) distance $\dtv{\mbig}{\msmall} = 1/3$, the expected number of accepted tokens\footnote{This is different from the number of generated token in one iteration, which is the number of accepted tokens plus one (corrected token).} from $\msmall$ with the token verification algorithm is $1 - 1/3 + (1 - 1/3)^2 = 10/9$ (see analysis in \cite{leviathan2022fast}).

\medskip
\noindent \textbf{An ideal algorithm with full information.} Suppose an algorithm can decide on what tokens to accept from $\msmall$ based on the \textit{full joint distributions} of both tokens, \ie
\[
    \mbig(\a\a) = 1/9, \quad
    \mbig(\a\b) = 2/9, \quad \mbig(\b\a) = 2/9, \quad \mbig(\b\b) = 4/9,
\]
\[
    \msmall(\a\a) = 4/9, \quad
    \msmall(\a\b) = 2/9, \quad \msmall(\b\a) = 2/9, \quad \msmall(\b\b) = 1/9.
\]
The algorithm would have performed the following improved acceptance logic: always accept $X_1X_2$ when $X_1X_2 = $ $\a\b$, $\b\a$, or $\b\b$ since $\mbig(X_1X_2) \ge \msmall(X_1X_2)$, and accept $\a\a$ with probability $\mbig(\a\a)/\msmall(\a\a) = 1/4$ (correcting the samples to $\b\b$). 
The expected number of accepted tokens from $\msmall$ now becomes: $2(\msmall(\a\b)+\msmall(\b\a)+\msmall(\b\b)+1/4 \times \msmall(\a\a))=12/9 > 10/9$. This illustrates the benefit of considering the distribution of draft blocks jointly.

\medskip
\noindent \textbf{Verification with partial information.} In general the full distribution over the next block of tokens is intractable to calculate. Instead, we only have access to the conditional distributions of the next token along the \emph{sample path} of the draft block, $\mbig(\cdot \mid \context, X^{i}), \msmall(\cdot \mid \context, X^{i})$ for various $i$'s. To emphasize, the ideal rejection logic does not need access to the full distribution, but care is needed in properly assigning the residual distribution. Our block verification does exactly this, as follows.

\confirm{For the simple toy example describe above, we propose the following improved algorithm, which is a simplified version of the general block verification algorithm stated in \cref{alg:block_verify}.} When the draft tokens $X_1X_2 = \a\b$ or $\b\b$, $\probof{\text{Accept } X_1X_2} = 1$ similar to the idealized algorithm. When $X_1X_2 = \a\a$, $\probof{\text{Accept } X_1X_2} = 1/4$, and else the algorithm rejects both tokens and only corrects the first token to $\b$ since the algorithm doesn't have access to $\mbig(\cdot \mid\b)$. When $X_1X_2 = \b\a$, it always accepts $\b$, and then accepts $\a$ with probability $1/2$ (else it corrects the second token to $\b$). Importantly, the marginal distributions of the generated tokens at the first token and the second token are always $\mbig(\cdot)$. \confirm{Moreover, the algorithm only uses distributions that are conditioned on the \emph{sample path} of the draft block, and hence it works in the partial information setting.} We then simply add the generated tokens to the prefix and proceed to the next iteration. 
The expected number of accepted tokens is $2\times (\msmall(\a\b)+\msmall(\b\b)) + (1 + 1/2) \times \msmall(\b\a)  + 1/4 \times 2 \times \msmall(\a\a)=11/9$, which is better than the $10/9$ obtained by token verification.
This example proves the following result:
\begin{lemma}\label{cor:non_optimal_speed}
    The standard token verification algorithm of speculative decoding is not optimal.
\end{lemma}

Note that while the expected number of accepted tokens in the example for block verification ($11/9$) is higher than that of the standard token verification algorithm ($10/9$), it is still less than that of the ideal algorithm with access to the full distribution ($12/9$).
In \cref{sec:guarantee_block}, we show that block verification is indeed optimal in the partial information case, with natural assumptions.

%% file: sections/algorithm.tex
\begin{figure}[h]
\begin{minipage}{0.46\textwidth}
\begin{algorithm}[H]
\caption{Token Verification}
\label{alg:token_verify}
\begin{algorithmic}[1]
\REQUIRE{Draft block $X^{\nt}$; small model distributions $\forall i< \nt, \msmall(\cdot \mid \context, X^{i})$; large model distributions $\forall i \le \nt, \mbig(\cdot \mid \context, X^{i})$.}
\STATE Sample $\eta_1, \ldots, \eta_{\nt} \sim U(0, 1)$.
\STATE Set $\tau = 0$.
\FOR{$i = 1, \ldots \nt$}
\item[]
\ifnum\ifarxiv=0
\item[]
\fi
\item[]
\STATE Set $\tt_i = \min\{\frac{\mbig(X_{i} \mid \context, X^{i-1})}{\msmall(X_{i} \mid \context,  X^{i-1})}, 1\}.$
\IF{$\eta_i \le \tt_i$}
\STATE Set $\tau = i$.
\ELSE
\STATE \textbf{break.}
\ENDIF
\ENDFOR
\IF{$\tau = \nt$}
\STATE Sample $Y$ from $\mbig(\cdot \mid \context, X^{\nt})$.
\ELSE
\STATE Sample $Y$ from $\pres^{\rm token}(\cdot \mid \context, X^{\tau})$ as in \cref{eqn:pres-token}.
\vspace{+2pt}
\ENDIF
\STATE \textbf{Return} $X^\tau, Y$.
\end{algorithmic}
\end{algorithm}
\end{minipage}
\hfill
\begin{minipage}{0.46\textwidth}
\begin{algorithm}[H]
\caption{Block Verification}
\label{alg:block_verify}
\begin{algorithmic}[1]
\REQUIRE{Draft block $X^{\nt}$; small model distributions $\forall i< \nt, \msmall(\cdot \mid \context, X^{i})$; large model distributions $\forall i \le \nt, \mbig(\cdot \mid \context, X^{i})$.}
\STATE Sample $\eta_1, \ldots, \eta_{\nt} \sim U(0, 1)$.
\STATE Set $\tau = 0$, \diff{$\pab_0 = 1$}.
\FOR{$i = 1, \ldots \nt$}
\STATE \diff{Set $\pab_i = \min \{\pab_{i-1}\frac{\mbig(X_{i} \mid \context,  X^{i-1})}{\msmall(X_{i} \mid \context, X^{i-1})}, 1 \}$.}
\STATE Set $\tb_i$ \diff{as in \cref{eqn:pacc_block}.}
\IF{$\eta_i \le \tb_i$}
\STATE Set $\tau = i$.
\ELSE
\STATE \diff{\textbf{continue.}} \label{line:continue}
\ENDIF
\ENDFOR
\IF{$\tau = \nt$}
\STATE Sample $Y$ from $\mbig(\cdot \mid \context, X^{\nt})$.
\ELSE
\STATE Sample $Y$ from $\pres^{\rm block}(\cdot \mid \context, X^{\tau})$ as in \diff{\cref{eqn:pres-block}.}
\ENDIF
\STATE \textbf{Return} $X^\tau, Y$.
\end{algorithmic}
\end{algorithm}\end{minipage}
\end{figure}

\section{Block Verification}\label{sec:algorithm}

In this section, we extend the above intuition to develop a general block verification algorithm, which works for standard speculative decoding with partial information. The high-level idea is to couple the acceptance of each draft token with other draft tokens. To do this, the algorithm considers draft sub-blocks with different lengths, and decides whether to accept each sub-block independently. The final accepted draft block is the longest accepted sub-block in the above process. The acceptance probabilities for each sub-block and the residual distributions are carefully chosen to maintain the distribution guarantee of the final output, and achieve optimal speedup.

\begin{figure}[h]
    \centering
   \begin{tcolorbox}[width=0.95\textwidth]
 Residual distribution in \cref{alg:token_verify} (Line 15): $
 \forall x \in \cX,$
 \begin{equation}
     \pres^{\rm token}(x \mid \context, X^{i}) = \frac{\max\{ \mbig(x \mid \context, X^{i}) - \msmall(x \mid \context, X^{i}), 0 \}}{\sum_{x'\in\cX} \mbig(x' \mid \context, X^{i}) - \msmall(x' \mid \context, X^{i}), 0 \}}.\label{eqn:pres-token}
 \end{equation}
 Residual distribution in \cref{alg:block_verify} (Line 15): $
 \forall x \in \cX,$
 \begin{equation}
     \pres^{\rm block}(x \mid \context, X^{i}) = \frac{\max\{ \diff{$\pab_{i}$} \cdot \mbig(x \mid \context, X^{i}) - \msmall(x \mid \context, X^{i}), 0 \}}{\sum_{x'\in\cX} \max\{ \diff{$\pab_{i}$} \cdot \mbig(x' \mid \context, X^{i}) - \msmall(x' \mid \context, X^{i}), 0 \}}.\label{eqn:pres-block}
 \end{equation}
 Acceptance probability in \cref{alg:block_verify} (Line 5): $\tb_\nt = \pab_\nt$, and when $i < \nt$,
 \begin{equation}\label{eqn:pacc_block}
    \tb_i = 
    \frac{\sum_{x' \in \cX} \max\{\pab_i \cdot \mbig(x' \mid \context, X^{i}) -\msmall(x' \mid \context, X^{i}), 0\}}{\sum_{x' \in \cX}  \max\{\pab_{i}\cdot \mbig(x' \mid \context, X^{i}) -\msmall(x' \mid \context, X^{i}), 0\} + 1 - \pab_i}.
\end{equation}
 \end{tcolorbox}
    \caption{The acceptance probabilities and residual distributions in \cref{alg:token_verify,alg:block_verify}.}
    \label{fig:missing_definition}
\end{figure}

See \cref{alg:block_verify} for a sketch implementation of block verification, and \cref{alg:token_verify} for a sketch implementation of the standard token verification for comparison.
Note that the implementations follow the same overall structure (the differences are highlighted).

\znew{Importantly, token verification stops as soon as a token is rejected (the \textbf{break} in Line~9 of \cref{alg:token_verify}), while block verification always operates on the full block. Equivalently, in token verification, $\tau = \argmin \{ \eta_i \le \tt_i\}$ while in block verification, $\tau = \argmax \{ \eta_i \le \tb_i\}$.}
\znew{This difference makes \specblockt tend to accept longer sub-blocks compared to \spectokent, resulting in higher block efficiencies. In \cref{fig:acceptance_length}, we plot the empirical complementary CDF of the acceptance length for both algorithms with the toy models introduced in \cref{eqn:toy_models} to demonstrate this.
}
\begin{figure}[h]
    \centering
    \includegraphics[width = 0.45\textwidth]{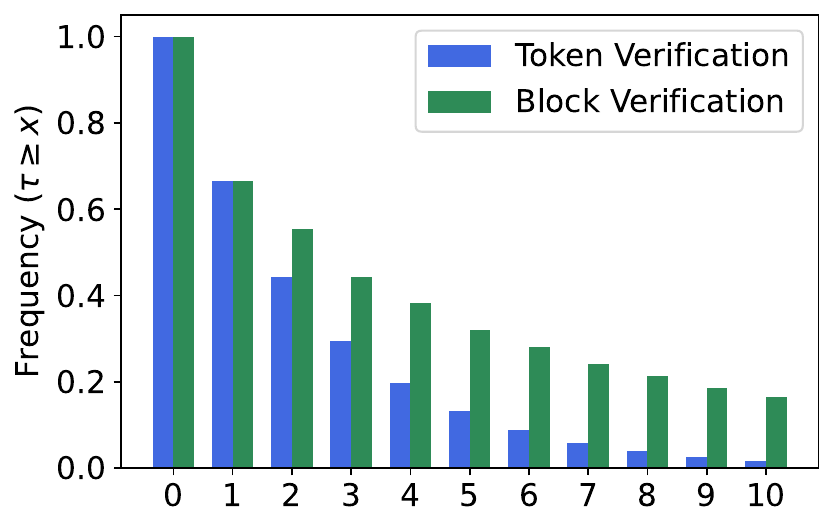}
    \caption{\znew{Empirical complementary CDF of $\tau$ for both algorithms with draft length $\nt = 10$. The draft and target models are the context-independent toy models introduced in \cref{eqn:toy_models}. }}
    \label{fig:acceptance_length}
\end{figure}

See \cref{alg:speculative_decoding_framework} for the outer loop of the speculative decoding algorithm, which remains unchanged for both verification methods.
See \cref{fig:missing_definition} for additional definitions. See \cref{sec:python_code} for sketch Python implementations. \znew{Due to the simplicity of the change, the algorithm can be easily implemented without incurring additional code complexity in practical systems.}

\begin{algorithm}[h]
\caption{Speculative decoding (\spec) \citep{leviathan2022fast}}
\label{alg:speculative_decoding_framework}
\begin{algorithmic}[1]
\REQUIRE{Prefix $\context$, target model $\mbig$, draft model $\msmall$. Draft length $\nt$. Verification algorithm $\draftverify$. }
\WHILE{$\eos \notin (X^\tau, Y)$}
\STATE Sample $X_1, \ldots, X_{\nt} \sim  \msmall(\cdot \mid \context)$ using autoregressive sampling, keep the conditional\\
probabilities at each step $\msmall(\cdot \mid \context, X^{i})$ for $i = 0, \ldots, {\nt-1}$. \hfill \coloredcomment{Obtain draft block.}
\STATE Call the large model $\mbig$ and compute conditional probabilities $\mbig(\cdot \mid \context, X^{i})$ \\
for $i = 0, 1, \ldots, \nt$ in parallel. \hfill \coloredcomment{Parallel scoring.}
\STATE Get the accepted tokens with draft verification \hfill \coloredcomment{Draft verification and correction.}
\[
   X^\tau, Y = \draftverify(X^{\nt}, \{\msmall(\cdot \mid \context, X^{i}) \}_{i = 0}^{\nt-1}, \{\mbig(\cdot \mid \context, X^{i}) \}_{i = 0}^{\nt}).
\]
\STATE $\context \leftarrow \context, X^\tau, Y.$ \hfill \coloredcomment{Add decoded tokens to the prefix.}
\ENDWHILE
\end{algorithmic}
\end{algorithm}

\paragraph{Theoretical guarantees.} Speculative decoding with block verification preserves the distribution of its outputs (\cref{thm:distribution_match}). Moreover, block verification achieves the {\em optimal} speedup among all valid draft verification algorithms in the outer loop of speculative decoding (\cref{alg:speculative_decoding_framework}), resulting in a strict improvement over the standard token verification (\cref{thm:optimal_decoded_length}). \confirm{We defer the formal statements and the intuitions the parameter choices in \cref{alg:block_verify} to \cref{sec:guarantee_block}.}

%% file: sections/theory_main.tex
\section{Theoretical Guarantees} \label{sec:guarantee_block}

In this section, we present the formal theoretical guarantees of block verification. Notably, that it produces the correct distribution and that it is optimal in terms of the expected number of generated tokens.
Let $\cM^*(\cdot \mid \context)$ denote the distribution of the sequence up to the end of the generative process under model $\cM$ and context $\context$.

\begin{definition}[Valid draft verification algorithm] \label{def:draft_verify}
\confirm{A draft verification algorithm \draftverify~takes the draft block $X^{\nt}$, small model distributions and large model distributions along the sample path, namely $\forall i< \nt, \msmall(\cdot \mid \context, X^{i})$ and $\forall i \le \nt, \mbig(\cdot \mid \context, X^{i})$ as inputs, and outputs a prefix $X^\tau, \tau \le \nt$ of $X^\nt$, and an additional token $Y$.} \draftverify~is said to be a valid draft verification algorithm if $\forall \context$, models $\msmall, \mbig$, and block length $\nt$, the outputs of \cref{alg:speculative_decoding_framework} (\spec) with verification algorithm~\draftverify~satisfy
    \begin{equation}
        \spec(\context, \mbig, \msmall, \nt, \draftverify) \simp \mbig^*(\cdot \mid \context)\footnote{We use $\simp$ to denote that two distributions are the same.}, \label{eqn:spec_goal}
    \end{equation}
\ie the distribution of the outputs is preserved.
\end{definition}

Note for example that the standard token verification algorithm is a valid draft verification algorithm (Appendix A.1 in \cite{leviathan2022fast}).

\pagebreak

We now claim the following:

\begin{theorem}\label{thm:distribution_match}
    Block verification \confirm{(\cref{alg:block_verify})} is a valid draft verification algorithm.
\end{theorem}

In other words, speculative decoding with block verification preserves the distribution of the output sequence.

We now further claim that block verification is optimal for all valid draft verification algorithms\footnote{We use \specblock and \spectoken to denote block verification and token verification respectively when convenient.}.

\begin{theorem}\label{thm:optimal_decoded_length}
For $i > 0$, let $N(i)$ be the number of decoded tokens after $i$ iterations in \cref{alg:speculative_decoding_framework}.
    For any valid draft verification algorithm \draftverify~in \cref{def:draft_verify}, we have $\forall \context, \msmall, \mbig$, $\nt$, and $i$,
    \[
        \EE_{\specblock}[N(i)] \ge  \EE_{\draftverify}[N(i)],
    \]
    where the randomness is over the randomness of the draft block and the randomness of the algorithm. 
    
    In particular,
    \[
     \EE_{\specblock}[N(i)] \ge  \EE_{\spectoken}[N(i)].
    \]
\end{theorem}

In other words, among all valid verification algorithms, speculative decoding with block verification decodes the highest number of tokens in expectation in a fixed number of iterations. Note that since the computation overhead added by block verification is negligibly small, this establishes the overall optimality of the block verification algorithm.
In particular, block verification provides a greater speedup than the standard token verification. We defer the proofs to \cref{sec:formal_guarantee}. \confirm{Below we give intuitions on the algorithm changes that contribute to achieving the above guarantees.}

\confirm{\noindent \textbf{Intuition on parameter choices and theoretical guarantees.} The key quantities for achieving the speedup and distribution matching guarantees are $\pab_i$'s. In \cref{lem:accept_prob} in \cref{proof:distribution_match}, we show that $\pab_i$ corresponds to the probability that the sub-block $X^{i}$ will be kept in the final output. This is guaranteed by choosing the per-step acceptance probability properly since block verification keeps the longest accepted sub-block. Next we discuss how $\pab_i$'s contribute to the distribution matching and optimality guarantees.}

\noindent \confirm{\textit{Distribution matching guarantee (\cref{thm:distribution_match}).} To start, we ignore the minimum operation in the recursive definition of $\pab_i$'s. In such case, each $\pab_i$ is simply $\mbig(X^i\mid\context) / \msmall(X^i\mid\context)$, which is an upper bound on the actual $\pab_i$'s.  As shown in \cref{lem:accept_prob}, for any $X^i$, the probability that it is in the accepted block is $\msmall(X^i) \pab_i(X^i)$. Since the draft block $X^i$ is generated with probability $\msmall(X^i\mid\context)$, this guarantees that the probability of getting $X^i$ by accepting it from the draft will be at most $\mbig(X^i\mid\context),$ and hence the algorithm is not accepting $X^i$ more than needed.
} 

\confirm{The remaining part is to choose a suitable residual distribution $\pres^{\rm block}$'s so that the distribution on the next token follows $\mbig(\cdot \mid X^i)$. Note that for any possible next token $x$,  $(X^i, x)$ could also be obtained by accepting $X^{i+1}$ when $X_{i+1} = x$, 
with a probability of $\msmall(X^{i}, x)\pab_{i+1}(X^{i}, x)$, 
which should be subtracted to obtain the residual mass on $(X^i, x)$. %
This leads to the choice of $\pres^{\rm block}$ in \cref{eqn:pres-block} after proper normalization.}

\confirm{\textit{Optimality guarantee (\cref{thm:optimal_decoded_length})}. For optimality guarantee, the main proof is to show that for any prefix $X^i$ in the draft block, $\pab_i(X^i)$ is the maximum probability that a valid verification algorithm can accept $X^i$, which is stated in \cref{lem:accept_prob_upper}. This implies that in one iteration, block verification accepts the most tokens in expectation, and the multi-iteration case can be obtained by an induction argument.}

\confirm{To see that why $\pab_i(X^i)$ is the upper bound on the acceptance probability, we show that this is necessary to guarantee that for any prefix $X^i$ that could be obtained from multiple draft sample paths, the distribution over subsequent tokens are always the same $\mbig(\cdot \mid \context, X^i)$.
This enables block verification to be used as a plug-and-play replacement of token verification in the outer loop of speculative decoding (\cref{alg:speculative_decoding_framework}).}

\pagebreak

Finally, we note that the optimality guarantee holds for all verification algorithms that can be used in \cref{alg:speculative_decoding_framework} as is. 
Specifically, there exist verification procedures that force the decoding logic to depend on the previous accept/reject decisions that produce more accepted tokens in average \emph{in one iteration}. However, this will affect the decoding speed in subsequent iterations. In \cref{sec:block_verify_greedy}, we present such an algorithm and name it \emph{greedy block verification}. We empirically observe that block verification consistently outperforms it, so we include it mainly as a theoretical result.

%% file: sections/experiments.tex
\section{Experiment Setup}

\confirm{We conduct experiments using PALM-2 models \citep{chowdhery2022palm}, and Vicuna models \citep{vicuna2023}.}

\confirm{For the experiments on PALM-2 models, we use PALM-2-S as the large target model and PALM-2-XXS / PALM-2-XXXS as the small drafter model.} The order of the sizes of the models is PALM-2-XXXS $<$ PALM-2-XXS $<$ PALM-2-S.
We evaluate on prompts from a wide range of datasets and tasks, including language modeling with one-billion language benchmark (LM1B) \citep{chelba2013one}, 
ChatGPT prompts sourced from LearnGPT (GPT Prompt) \citep{chatgptprompts}, %
reasoning questions (WebQA) \citep{berant-etal-2013-semantic}, physical commonsense reasoning questions (PIQA) \citep{Bisk_Zellers_Le-bras_Gao_Choi_2020}, scraped conversations with ChatGPT (ShareGPT) \citep{chatgptprompts, sharegpt}, summarization tasks (XSum) \citep{Narayan2018DontGM}, grade school math problems (GSM8K) \citep{cobbe2021gsm8k}, and German to English translation (WMT DeEn) \citep{bojar-EtAl:2014:W14-33}.
For all datasets, we decode the first $1000$ prompts using a max input prompt length of $512$ and decode up to $128$ output tokens. \confirm{We use a batch size of 1 in all experiments except for the experiments in \cref{app:multi-draft}. Note that since our method only modifies the verification phase of the algorithm and doesn't introduce additional draft tokens, the speedup we get is independent of the batch size. We use a temperature of 1.0 for the experiments on PALM-2 models.}

\confirm{For Vicuna family of models~\citep{vicuna2023}, we conduct the set of experiments in Spec-Bench \citep{xia-etal-2024-unlocking}. We discussed detailed settings for these expereiments in \cref{sec:exp_specbench}.}
\section{Results} \label{sec:experiments}

\znew{We focus our main experiments on the comparison between block verification and token verification.} \znew{Recent works \citep{sun2023spectr, miao2023specinfer} have extended speculative decoding to the case with multiple draft blocks to improve block efficiency. However, these methods also increase the required computation from the large model to verify the drafts, which is undesirable when query batching is performed. We empirically show that our method outperforms these methods in the large batch setting even with only one draft block. We defer the results to \cref{app:multi-draft} and focus on the one draft case in the main section below.}

\subsection{Experimental results on PALM-2 models}

We empirically compare speculative decoding with block verification to speculative decoding with token verification, and find that block verification provides small yet consistent improvements in a wide range of settings, both when measuring idealized \emph{block efficiency} and real world \emph{wall clock time}.

\emph{Block efficiency} measures the speedup in an idealized settings where we neglect the evaluation time of the draft model and assume that we have enough compute capacity for evaluating the large model on all draft prefixes in parallel. Specifically, it measures the average number of decoded tokens per serial call to the target model. We observe consistent improvements for all datasets and draft models.
For $\nt = 8$ with PALM-2-XXS as the drafter, the improvement in block efficiency ranges from $7.00\%$ to $10.06\%$ with an average of $8.30\%$. 

We also observe consistent improvements in \emph{wall clock time}, which measures the actual speedup, including all the real-world overheads.
See \citep{leviathan2022fast, chen2023accelerating} for a more detailed discussion of these overheads.
For $\nt = 8$ with PALM-2-XXS as the drafter, the improvement in block efficiency ranges from $5.36\%$ to $8.14\%$ with an average of $6.49\%$. The detailed numbers for this setting are listed in \cref{tab:exp_comparison_l_8}.

\begin{table*}[h]
\caption{Speedup comparison between token verification (\textsc{TokenV}) and block verification (\textsc{BlockV}) with $\nt = 8$ \confirm{with PALM-2-S as the target model and PALM-2-XXS as the draft model} on various datasets and tasks. We list the average and standard deviation across 3 runs with different seeds on $1000$ test prompts.}
 \setlength{\tabcolsep}{2pt}
\begin{center}
\begin{tabular}{  c c c c c c c}
\toprule
\multirow{2}{*}{Dataset} &  \multicolumn{3}{c}{Block efficiency} & \multicolumn{3}{c}{Wall clock time speedup over baseline} \\
\cmidrule(lr){2-4} \cmidrule(lr){5-7} 
& \textsc{TokenV} & \textsc{BlockV}   & Improve. $\uparrow$ \% & \textsc{TokenV} & \textsc{BlockV}  & Improve. $\uparrow$ \% \\ 
 \midrule
LM1B & $3.21\pm 0.01$ & $3.49\pm 0.02$ & $8.68\pm 0.79$ & $2.17\pm 0.01$ & $2.32\pm 0.01$ & $6.85\pm 0.74$ \\
GPT Prompt & $3.41\pm 0.04$ & $3.76\pm 0.02$ & $10.06\pm 1.66$ & $2.30\pm 0.02$ & $2.48\pm 0.01$ & $8.14\pm 1.55$ \\
WebQA & $3.44\pm 0.01$ & $3.70\pm 0.01$ & $7.53\pm 0.24$ & $2.32\pm 0.00$ & $2.45\pm 0.01$ & $5.75\pm 0.22$ \\
PIQA & $3.40\pm 0.02$ & $3.68\pm 0.00$ & $8.30\pm 0.62$ & $2.29\pm 0.01$ & $2.44\pm 0.00$ & $6.52\pm 0.58$ \\
ShareGPT & $3.34\pm 0.01$ & $3.62\pm 0.03$ & $8.45\pm 0.98$ & $2.25\pm 0.01$ & $2.40\pm 0.02$ & $6.68\pm 0.91$ \\
XSum & $3.49\pm 0.02$ & $3.76\pm 0.01$ & $7.63\pm 0.94$ & $2.35\pm 0.01$ & $2.49\pm 0.01$ & $5.82\pm 0.88$ \\
GSM8K & $3.81\pm 0.01$ & $4.15\pm 0.03$ & $8.74\pm 0.56$ & $2.55\pm 0.01$ & $2.73\pm 0.02$ & $6.84\pm 0.51$ \\
WMT-DeEn & $3.19\pm 0.01$ & $3.41\pm 0.02$ & $7.00\pm 0.78$ & $2.15\pm 0.01$ & $2.27\pm 0.01$ & $5.36\pm 0.73$ \\
\midrule
Average & $3.41$ &	$3.70$ & $8.30$ & $2.30$ &	$2.45$  &$6.49$ \\
  \bottomrule
\end{tabular}
\end{center}
\label{tab:exp_comparison_l_8}
\end{table*}

\paragraph{The effect of draft length $\nt$.} We also perform comparisons of the algorithms for other block lengths ($\nt = 4$ and $\nt = 6$) and observe consistent improvements. We plot the average improvement over all datasets in \cref{fig:average_improvement} with the numbers in \cref{tab:avg_improvement}. With the same drafter, the relative improvement of block verification over token verification increases as $\nt$ increases. This is consistent with our intuition since when $\nt = 1$, the two algorithms are the same and as $\nt$ increases, block verification would benefit more from coordinating the acceptance rule considering the realization of all tokens in the draft block. 

\znew{As shown in \cref{tab:avg_improvement}, similar to \spectokent, the block efficiency of \specblockt increases as $\nt$ increases. However, the wall clock speedup peaks at a certain draft length ($\gamma = 4$ or $\gamma = 6$ for all settings) due to the increased computation cost in the verification phase. Hence we focus on $\gamma \le 8$ in our experiments.}

\paragraph{The effect of the drafter.} We also consider the effect of the quality of the drafter on the improvement. In \cref{tab:avg_improvement}, we list the average block efficiency and wall clock speed up under different draft lengths for both drafters. Note that PALM-2-XXS is a larger model than PALM-2-XXXS, and hence a better drafter in terms of quality, as demonstrated by the better average block efficiencies in the table.
In \cref{fig:average_improvement}, we plot the average improvement under different drafter models, PALM-2-XXS and PALM-2-XXXS. The improvements hold for both drafters. And the relative improvement in block efficiency under PALM-2-XXS  is greater than that under PALM-2-XXXS.
This shows that the improvement obtained from block verification can be combined with the improvement on the quality of the drafter, and the improvement might be more significant under better drafters.
    
\begin{figure}[h]
        \begin{minipage}{0.45\textwidth}
    \centering
\setlength{\tabcolsep}{3pt}
\begin{tabular}{ c c  c  c  c c }
\toprule
\multirow{2}{*}{$\nt$} & \multirow{2}{*}{Drafter} & \multicolumn{2}{c}{\textsc{TokenV}}  & \multicolumn{2}{c}{\textsc{BlockV}}  \\ 
 & & BE  & WS  & BE  & WS \\ 
 \midrule
\multirow{2}{*}{$4$ } & XXS & $2.89$ & $2.44$ & $2.99$ & $2.50$ \\
& XXXS & $2.35$ & $2.36$ & $2.43$ & $2.43$ \\ 
\midrule 
\multirow{2}{*}{$6$ }& XXS & $3.23$ & $2.43$ & $3.43$ & $\mathbf{2.54}$ \\
& XXXS & $2.50$ & $2.39$ & $2.63$  & $2.50$ \\
\midrule
\multirow{2}{*}{$8$ }& XXS & $3.41$ & $2.30$ & $\mathbf{3.70}$ & $2.45$ \\
& XXXS & $2.57$ & $2.28$ & $2.73$ & $2.40$ \\
  \bottomrule
\end{tabular}
  \caption{Table on average block efficiency (BE) and wall clock speedup (WS) across all datasets for token verification (\textsc{TokenV}) and block verification (\textsc{BlockV}) with different $\nt$.  The large model is PALM-2-S and the drafter model is either PALM-2-XXS (XXS) or PALM-2-XXXS (XXXS).}
  \label{tab:avg_improvement}
    \end{minipage}
        \hfill
    \begin{minipage}{0.45\textwidth}
        \centering
    \includegraphics[width = 0.95\textwidth]{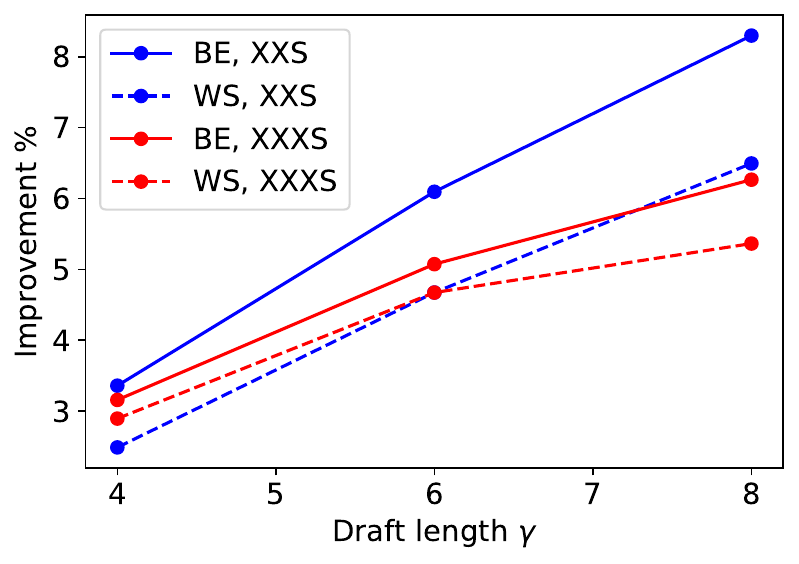}
    \caption{Average relative improvement of block verification over token verification in block efficiency (BE) and wall clock speedup (WS) across all datasets for different drafters and draft lengths.}
    \label{fig:average_improvement}
    \end{minipage}
\end{figure}

Detailed results for experiments performed with different drafters, different datasets, and different draft lengths are listed in \cref{sec:experiments_app}.

\subsection{\confirm{ Experimental results on Spec-Bench with Vicuna models}} \label{sec:exp_specbench}

\confirm{We also conduct the set of experiments proposed in Spec-Bench~\citep{xia-etal-2024-unlocking} with Vicuna family of models \citep{vicuna2023}. The benchmark includes various generation subtasks including multi-turn conversation,
retrieval-augmented generation, summarization, translation, question answering, and
mathematical reasoning. See \cite{xia-etal-2024-unlocking} for a detailed discussion of the subtasks. 
For all experiments in this section, we use a single NVIDIA H100 GPU with a batch
size of 1 and a max generation length of $1024$. We use Vincuna-7B-v1.3 as the target model and Vincuna-68M as the draft model. To study the effect of temperature, we consider temperatures in $\{0.2, 0.6, 1.0\}$ and fix $\gamma = 8$. The results are listed in \cref{tab:exp_spec_bench}. The reported numbers are the average of 3 runs.}

\confirm{Our algorithm obtains consistent improvement compared to token verification (up to 8.7\% in block efficiency and up to 6.7\% in wall clock speedup) across different draft lengths for all temperatures bigger than 0. This demonstrates the applicability of our method for different families of models.}

\confirm{\paragraph{The effect of temperature.} Note that for temperature of 0, which corresponds to greedy decoding, our algorithm degenerates to token verification and doesn't provide additional speedups. In non-zero temperature settings, the advantage is consistent and the additional improvement is higher for larger temperatures. The observation is consistent with the intuition behind the algorithm, which  obtains improvement on block efficiency by coordinating the randomness in the acceptance decisions at different token locations. %
}

\begin{table*}[h]
\caption{\confirm{Speedup comparisons between token verification (\textsc{TokenV}) and block verification (\textsc{BlockV}) on Spec-Bench \citep{xia-etal-2024-unlocking} for temperature $T \in \{0.2, 0.6, 1.0\}$.
We use Vicuna-7B-v1.3 as the target model and Vicuna-68M as the draft model. $\gamma = 8$ for all experiments and each number is an average of 3 runs.}}
 \setlength{\tabcolsep}{2pt}
\begin{center}
\begin{tabular}{ c c c c c c c}
\toprule
 \multirow{2}{*}{$T$} &  \multicolumn{3}{c}{Block efficiency} & \multicolumn{3}{c}{Wall clock speedup over baseline} \\
\cmidrule(lr){2-4} \cmidrule(lr){5-7} 
& \textsc{TokenV} & \textsc{BlockV}   & Improve. $\uparrow$ \% & \textsc{TokenV} & \textsc{BlockV}  & Improve. $\uparrow$ \% \\ 
\midrule
0.2 & 2.75  & 2.85 & 3.72 & 1.22 & 1.24 & 1.66  \\
0.6 & 2.75  & 2.90 & 5.32 & 1.23 & 1.29 & 4.24  \\
1.0 & 2.79  & 3.04 & 8.70 & 1.27 & 1.34 & 6.07  \\
  \bottomrule
\end{tabular}
\end{center}
\label{tab:exp_spec_bench}
\end{table*}

%% file: sections/related.tex
\section{Related work} \label{sec:related}

\textbf{Parallel decoding.} Our work improves speculative decoding \citep{leviathan2022fast}, a framework for decoding several tokens concurrently. \emph{Draft and verify} \citep{stern2018blockwise} was an earlier work, which proposed to independently predict and decode several tokens in parallel, for the greedy decoding case (zero temperature). Speculative decoding has later also been proposed in \cite{chen2023accelerating}.

\begin{table*}[h]
\centering
\caption{Recent works based on the draft and verify framework.
Temperature 0 refers to greedy decoding and non-zero temperature refers to sampling.}

 \setlength{\tabcolsep}{3pt}
 \vspace{+10pt}
\begin{tabular}{ c c c c c }
 \toprule
Work & \# drafts & Temp. & Drafting & Verification \\
 \midrule
 \cite{stern2018blockwise}& $1$ & $0$ & parallel softmax layers  & token matching \\ 
 \cite{yang2023inference} & $1$ & $0$ & additional 
 text  &  token matching \\
 \midrule
 \cite{leviathan2022fast}& $1$ & $\geq 0$ & small LM  & \spectoken~ (\cref{alg:token_verify})  \\
  \cite{chen2023accelerating} & $1$ & $\geq 0$& small LM  & \spectoken~ (\cref{alg:token_verify})  \\
  \cite{he2023rest}  & $1$ & $\geq 0$ & database retrieval  & \spectoken~ (\cref{alg:token_verify})  \\
  \cite{chen2023cascade}& $1$ & $\geq 0$ & cascade of small LMs & \spectoken~ (\cref{alg:token_verify})  \\
  \cite{sun2024triforce} & $1$ & $\geq 0$ & hierarchical drafters & \spectoken~(\cref{alg:token_verify}) \\
  \cite{zhou2023distillspec} & $1$ & $\geq 0$ & distilled small LMs & \spectoken~(\cref{alg:token_verify})\\
  \cite{liu2023online} & $1$ & $\geq 0$ & distilled small LMs & \spectoken~(\cref{alg:token_verify})\\
   \cite{gloeckle2024better} & $1$ & $\geq 0$ & parallel softmax layers & \spectoken~(\cref{alg:token_verify})\\
   \cite{zhang2024draft} & $1$ & $\geq 0$ & layer skip & \spectoken~(\cref{alg:token_verify})\\
    \cite{elhoushi2024layerskip} & $1$ & $\geq 0$ & early exit  & \spectoken~(\cref{alg:token_verify})\\
    \midrule
    {\bf This work}& $1$ & $\geq 0$ & small LM &  \specblock~ (\cref{alg:block_verify}) \\
 \midrule
 \cite{sun2023spectr} & $\geq 2$ & $\geq 0$ & small LM  &   SpecTr \\
  \cite{miao2023specinfer}  & $\geq 2$ & $\geq 0$& small LM  &   multi-round  \spectoken\\
  \cite{li2024eagle} & $\geq 2$ & $\geq 0$& small LM  &   multi-round  \spectoken\\
  \cite{chen2024sequoia} & $\geq 2$ & $\geq 0$& small LM  &   multi-round  \spectoken \\
  \bottomrule
\end{tabular}
\label{tab:related}
\end{table*}

\medskip
\noindent\textbf{Single draft improvements.}
There have been many works aiming to improve speculative decoding without making use of more than one decoding draft. In Table~\ref{tab:related}, we list a set of works in the draft and verify framework with a breakdown of their drafting and verification algorithms. See \cite{xia-etal-2024-unlocking} for a comprehensive study.  
In the single-draft case, %
 several works have worked on 
improving the drafting phase of speculative decoding \citep{he2023rest, chen2023cascade, sun2024triforce, zhou2023distillspec, liu2023online, gloeckle2024better, zhang2024draft, elhoushi2024layerskip}.
However, these algorithms all use the same token verification algorithm. Our proposed block verification algorithm can be used in tandem with the drafting techniques in \cref{tab:related}, yielding combined gains.
We leave a more systematic study of the improvement of block verification in these cases for future study.

\znew{The only other work that we are aware of that improves the verification step in speculative decoding is the independent work of \cite{hu2024accelerated}, which uses tree Monte Carlo to improve speculative decoding in the single draft case, and have proved that their algorithm improves over token verification. On the contrary, we prove that our algorithm achieves the optimal speedup among \revise{all valid verification algorithms}, including theirs.  Our  algorithm also requires minimal changes to the original token verification algorithm, making it easy to implement and adapt everywhere in practice.}

\medskip
\noindent \textbf{Multiple drafts.} Recently, speculative decoding is extended to multiple drafts \citep{sun2023spectr, miao2023specinfer} and new verification algorithms for the multi-draft scenario are proposed \citep{li2024eagle, chen2024sequoia}. While increasing the number of draft sequences has shown to improve the overall speedup, it comes at the cost of more computation. \znew{In \cref{app:multi-draft}, we show that in the large batch setting, where the inference is less memory bound, our method outperforms these methods without increasing the number of draft blocks.} In all of these works, the verification algorithm is a generalization of the token verification procedure. Extending block verification to the multi-sample case is an interesting future direction.

%% file: sections/implementation.tex
\section{Python Implementation}
\label{sec:python_code}

In this section we provide a sketch implementation of block verification (\cref{alg:block_verify}) in Python. 
Note that these are meant for illustration purposes only and are not fit for practical use.

Let $V = |\cX|$ be the size of the vocabulary.
The inputs to the algorithm are:
\begin{itemize}
    \item \texttt{ps}: an $(\nt+1)\times V$ \texttt{numpy} array with the distributions from the large model $\mbig(\cdot \mid \context, X^{i})$;
    \item \texttt{qs}: an $\nt \times V$ \texttt{numpy} array with the distributions from the draft model $\msmall(\cdot \mid \context, X^{i})$;
    \item \texttt{drafts}: a length-$\nt$ \texttt{numpy} array with the ids of the draft tokens $X^\nt$;
\end{itemize}

\begin{minted}[frame=lines, framesep=2mm]{python}
def block_verification(
  ps: np.ndarray, qs: np.ndarray, drafts: np.ndarray) -> list[int]:
  draft_length, vocab_size = qs.shape
  qs.resize((draft_length+1, vocab_size)) # Append a zero vector
  token_sequence = [] # Will include the token sequence we return
  accept_probability = 1.0 # Acceptance prob. for each sub-block
  probability_ratios = ps / qs
  # Add one token to indicate rejecting the sequence
  vocab_plus_one = np.arange(vocab_size + 1)
  for token_index, token_value in enumerate(xs):
    # Unnormalized residual probability
    sampling_weights[:vocab_size] = np.maximum(
        0, ps[token_index] * accept_probability - qs[token_index])
    # Unnormalized probability of rejecting the sequence
    sampling_weights[vocab_size] = 1 - accept_probability
    sampling_weights /= np.sum(sampling_weights)
    chosen_token = np.random.choice(vocab_plus_one,
                                    p=sampling_weights)
    # Update the sequence
    if chosen_token < vocab_size:
      token_sequence = xs[:token_index] + [chosen_token]
    # Update the acceptance probability
    accept_probability = min(1, probability_ratios[
        token_index, token_value] * accept_probability)
  return token_sequence
\end{minted}

For reference, here is a sketch implementation of the token verification algorithm (\cref{alg:token_verify}):
\begin{minted}[frame=lines, framesep=2mm]{python}
def token_verification(
  ps: np.ndarray, qs: np.ndarray, drafts: np.ndarray) -> list[int]:
  draft_length, vocab_size = qs.shape
  qs.resize((draft_length+1, vocab_size)) # Append a zero vector.
  token_sequence = [] # Will include the token sequence we return
  probability_ratios = ps / qs
  token_index = 0
  vocab_range = np.arange(vocab_size)
  for token_value in xs:
    accept_probability = probability_ratios[token_index, token_value]
    if (not np.isfinite(accept_probability) or 
        np.random.random() > accept_probability): # Rejection
      break
    token_index += 1
    token_sequence.append(token_value)
  # Calculate the residual distribution
  sampling_weights = np.maximum(0, ps[token_index] - qs[token_index])
  sampling_weights /= np.sum(sampling_weights)
  token_sequence.append(np.random.choice(vocab_range,
                                         p=sampling_weights))
  return token_sequence
\end{minted}

%% file: sections/theory.tex
\section{Formal Proofs} \label{sec:formal_guarantee}

We start by setting up a few necessary notations. Let $\cX$ be the space of output tokens. %
For $\ell > 1$, we use $\cM^\ell(\cdot \mid \context)$ to denote the joint distribution of the next $\ell$ tokens conditioned on the prefix under $\cM$, \ie for all $x_{1}, \ldots x_{\ell} \in \cX^\ell$,
$
    \cM^\ell(x_{1}, \ldots, x_{\ell} \mid \context) = \prod_{i = 1}^\ell \cM(x_{i} \mid \context, x^{i-1}).
$
We use $\cM^*(\cdot \mid \context)$ to denote the distribution of the sequence up to the end of the generative process. %
Below we first describe a necessary and sufficient condition for a valid draft verification algorithm in \cref{alg:speculative_decoding_framework}. %

\begin{lemma} \label{lem:verify_goal}
    $\forall \context, \msmall, \mbig$, $\nt$, let $X^\nt$ be generated from $\msmall^\nt(\cdot \mid \context)$, and
\[
   X^\tau, Y = \draftverify(X^\nt, \{\msmall(\cdot \mid \context, X^{i}) \}_{i = 0}^{\nt-1}, \{\mbig(\cdot \mid \context, X^{i}) \}_{i = 0}^{\nt}).
\]
Let $Z^{\nt-\tau}$ be generated from   $\mbig^{\nt-\tau}(\cdot \mid \context, X^\tau, Y)$.

 $\draftverify$ is a valid draft verification algorithm (\cref{def:draft_verify}) if and only if $\forall \context, \msmall, \mbig$, $\nt$,
    \begin{equation} \label{eqn:verify_goal}
        X^\tau, Y, Z^{\nt-\tau} \simp \mbig^{\nt+1}(\cdot \mid \context).
    \end{equation}
\end{lemma}

\begin{proof}
We first prove the forward direction (\cref{eqn:verify_goal} implies that $\draftverify$ satisfies \cref{def:draft_verify}) by induction on the maximum generation length of $\mbig(\cdot \mid \context)$. When the maximum generation length is $0$, for all new context $\context'$, we have the next token is a point mass over $\eos$, \ie
\[
    \mbig(x \mid \context, \context') = \delta \{x = \eos\}.
\]

Then \cref{eqn:verify_goal} implies that \draftverify~will only output $\eos$, which is the same as \cref{def:draft_verify}. Suppose \cref{eqn:spec_goal} holds for all context and $\mbig$ with generation length at most $T$, for a context $\context$ and $\mbig$ with maximum generation length at most $T + 1$, we have that the output of $\spec(\context, \mbig, \msmall, \nt, \draftverify)$ is 
\[
    X^\tau, Y, \spec((\context, X^\tau, Y), \mbig, \msmall, \nt, \draftverify).
\]

Let $Z^{\nt - \tau}$ be the first $\nt - \tau$ tokens from $\spec((\context, X^\tau, Y), \mbig, \msmall, \nt, \draftverify)$, and $O^{*}$ be the tokens after. Since $X^\tau, Y$ is at least of length one, the generation length of $\mbig(\cdot \mid \context, X^\tau, Y)$ is at most $T$. By the induction hypothesis, we have
\[
    Z^{\nt - \tau} \simp \mbig^{\nt-\tau}(\cdot \mid \context, X^\tau, Y),
\]
and
\[
  O^{*} \simp   \mbig^{*}(\cdot \mid \context, X^\tau, Y, Z^{\nt-\tau}).
\]
And hence by \cref{eqn:verify_goal},
\begin{align*}
   \spec(\context, \mbig, \msmall, \nt, \draftverify) & =  X^\tau, Y, \spec((\context, X^\tau, Y), \mbig, \msmall, \nt, \draftverify)  \\
   & = X^\tau, Y, Z^{\nt - \tau}, O^{*}\\
    & \simp  \mbig^{*}(\cdot \mid \context).
\end{align*}

This completes the proof for the forward direction.

For the backward direction, we have \cref{eqn:spec_goal} implies that for all $X^\tau, Y$,
\[
    \spec((\context, X^\tau, Y), \mbig, \msmall, \nt, \draftverify)[:\nt-\tau]\footnote{We use ${\bm v}[i:j]$ to denote the entries $i$ to $j$ in ${\bm v}$.} \simp \mbig^{\nt-\tau}(\cdot \mid \context, X^\tau, Y).
\]
Let $Z^{\nt - \tau}$ be a draw from $\mbig^{\nt-\tau}(\cdot \mid \context, X^\tau, Y)$, then
\[
Z^{\nt - \tau} \simp \spec((\context, X^\tau, Y), \mbig, \msmall, \nt, \draftverify)[:\nt-\tau].
\]
And hence when $X^\tau, Y$ is the output of $\draftverify$,
\begin{align*}
    X^\tau, Y, Z^{\nt - \tau} & \simp X^\tau, Y, \spec((\context, X^\tau, Y), \mbig, \msmall, \nt, \draftverify)[:\nt-\tau] \\
    & \simp \spec(\context, \mbig, \msmall, \nt, \draftverify)[:\nt+1] \\
    & \simp \mbig^{\nt+1}(\cdot \mid \context),
\end{align*}
\znew{where the last derivation follows from \cref{eqn:spec_goal} in \cref{def:draft_verify}.}
\end{proof}

In all proofs below, we fix the context $\context$, and the models $\msmall$ and $\mbig$. We note that the proofs won't use specific information about these choices and hence can be easily extended to all cases.

\subsection{Proof of \cref{thm:distribution_match}} \label{proof:distribution_match}
\input{sections/proof_distribution_match}

\subsection{Proof of \cref{thm:optimal_decoded_length}} \label{proof:optimal_decoded_length}
\input{sections/proof_thm_optimal}

\subsection{Proof of \cref{lem:accept_prob}} \label{sec:pab_acc_prob_proof}
\input{sections/proof_lem_acc_prob}

\subsection{Proof of \cref{lem:accept_prob_upper}} \label{sec:proof_lem_prob_upper}
\input{sections/proof_lem_prob_upper}

%% file: sections/proof_distribution_match.tex
By \cref{lem:verify_goal}, it would be enough to prove that \specblockt satisfies \cref{eqn:verify_goal}.
For simplicity, we often refer to the sequence $(X^\tau, Y, Z^{\nt - \tau})$ by $O^{\nt + 1}$. Note that $O_{\nt+1} \sim \mbig(\cdot \mid \context, O^\nt)$ always holds since when $\tau < \nt$, $O_{\nt+1} \sim \mbig(\cdot \mid \context, O^\nt)$ by definition and when $\tau = \nt$, $O_{\nt+1} = Y \sim \mbig(\cdot \mid \context, X^\nt) = \mbig(\cdot \mid \context, O^\nt)$. Hence it is enough to prove the following
\begin{equation} \label{eqn:output_block_distribution}
   \forall \ell \le \nt, \forall x^{\ell} \in \cX^\ell, \quad \probof{O^{\ell} = x^{\ell}}  = \mbig^\ell(x^{\ell} \mid \context),
\end{equation}

Note that in \specblockt, $\pab_i$'s depend on the draft tokens $X^\nt$. The following definition makes this explicit. Let $\pab_i$ be such that $\pab_0 = 1$, and $\forall 1 \le i \le \nt$, $x^i \in \cX^i$,
\begin{equation}
    \pab_i(x^i \mid \context) = \min \left \{ \pab_{i-1}(x^i \mid \context)\frac{\mbig(x_i \mid \context, x^{i-1})}{\msmall(x_i \mid \context, x^{i-1})}, 1\right \}. \label{eqn:def_as}
\end{equation}
For most cases, when the prefix $\context$ is clear, we will ignore $\context$ and simply use $\pab_{i}(x^i) = \pab_i(x^i \mid \context)$. We will only make the prefix explicit when necessary.

We first state the following lemma on the distribution of the number of tokens accepted by \specblockt.
\begin{lemma} \label{lem:accept_prob}
Let $X^\nt \sim \msmall^\nt(\cdot \mid \context)$, and 
\[
      X^\tau, Y = \specblock(X^\nt, \{\msmall(\cdot \mid \context, X^{i}) \}_{i = 0}^{\nt-1}, \{\mbig(\cdot \mid \context, X^{i}) \}_{i = 0}^{\nt}).
\]
Then we have $\forall i \le \nt$, and $x^i \in \cX^i$,
\[
    \probof{\tau \ge i \mid X^i = x^i} = \pab_i (x^i).
\]
\end{lemma}

We first prove \cref{thm:distribution_match} based on \cref{lem:accept_prob} and defer the proof of the lemma to \cref{sec:pab_acc_prob_proof}.
We prove \cref{eqn:output_block_distribution} by induction on the time index $\ell$.
When $\ell = 1$, $O_1$ is either $X_1$, or a residual sample from $\pres^{\rm block}(\cdot \mid \context)$ where $$\pres^{\rm block}(\cdot \mid \context) = \frac{\max\{ \mbig(x \mid \context ) - \msmall(x \mid \context ), 0\}}{\sum_{x'} \max\{ \mbig(x' \mid \context ) - \msmall(x' \mid \context), 0\}}, $$  
Hence we have $\forall x \in \cX$, by \cref{lem:accept_prob},
\begin{align}
& \quad \probof{O_1 = x}  \nonumber \\
    & = \probof{O_1 = x, \tau \ge 1} + \probof{O_1 = x, \tau = 0} \nonumber\\
    & = \probof{X_1 = x} \probof{\tau \ge 1 \mid X_1 = x} + \sum_{x'} \probof{X_1 = x'} (1 - \probof{\tau \ge 1 \mid X_1 = x'}) \cdot \pres^{\rm block}(x \mid \context)  \nonumber\\
    & = \msmall(x \mid \context) \cdot \pab_1(x)  + \sum_{x'} \msmall(x' \mid \context)  (1 - \pab_1(x')) \cdot \pres^{\rm block}(x \mid \context)  \nonumber\\
    & = \min \{ \mbig(x \mid \context ), \msmall(x \mid \context )\} + \sum_{x'} \max\{ \mbig(x'\mid \context) - \msmall(x'\mid \context), 0\} \cdot \pres^{\rm block}(x \mid \context)  \label{eqn:def_p1}\\
    & = \min \{ \mbig(x\mid \context), \msmall(x\mid \context)\} + \max\{ \mbig(x \mid \context ) - \msmall(x\mid \context), 0\}  \label{eqn:residual_sum}\\
    & =  \mbig(x\mid \context)  \nonumber,
\end{align}

\znew{where \cref{eqn:def_p1} comes from the definition of $\pab_1$ in \cref{eqn:def_as} and \cref{eqn:residual_sum} is due to \cref{eqn:pres-block} with $i = 0$.}
Hence the \cref{eqn:output_block_distribution} holds for $\ell = 1$. Suppose \cref{eqn:output_block_distribution} holds up to  $\ell < \nt$. For $\ell = \ell + 1$, we have $O_{\ell+1}$ is either equal to $X_{\ell+1}$ when $\tau \ge \ell + 1$, or a sample from $ \pres^{\rm block}(\cdot \mid \context, X^{\ell})$ when $\tau = \ell$, or a sample from $\mbig(\cdot \mid \context, O^{\ell})$ when $\tau < \ell$. Hence $\probof{O^{\ell+1} = x^{\ell + 1}}$ can be broken down below:
\begin{align}\label{eqn:break_down}
    & \quad \probof{O^{\ell+1} = x^{\ell + 1}} =  \nonumber \\
    & \probof{O^{\ell+1} = x^{\ell + 1}, \tau \ge \ell + 1} +  \probof{O^{\ell+1} = x^{\ell + 1}, \tau =  \ell} + \probof{O^{\ell+1} = x^{\ell + 1}, \tau < \ell}
\end{align}

For the first term ($\tau \ge \ell+1$), we have
\begin{align}
    & \quad \probof{O^{\ell+1} = x^{\ell+1}, \tau \ge \ell + 1} \nonumber \\ & =  \probof{X^{\ell+1} = x^{\ell+1}} \cdot \probof{\tau \ge \ell + 1 \mid X^{\ell+1} = x^{\ell+1}}  \nonumber \\
    & = \msmall(x^{\ell+1} \mid \context) \cdot \pab_{\ell+1}(x^{\ell+1}) \nonumber \\
    & = \msmall(x^{\ell} \mid \context) \cdot \min\{ \pab_{\ell}(x^{\ell}) \mbig(x_{\ell+1}\mid \context, x^{\ell}), \msmall(x_{\ell+1}\mid \context, x^{\ell})\}. \label{eqn:all_accept}
\end{align}

For the second term ($\tau = \ell$), we have
\begin{align*}
    & \quad \probof{O^{\ell+1} = x^{\ell+1}, \tau = \ell} \\
    & =  \probof{X^{\ell} = x^{\ell}} \cdot \probof{\tau = \ell \mid X^{\ell} = x^{\ell}} \cdot \probof{O_{\ell + 1} = x_{\ell + 1}\mid O^{\ell} = x^{\ell}, \tau = \ell} \\
    & = \probof{X^{\ell} = x^{\ell}} \cdot \probof{\tau = \ell \mid X^{\ell} = x^{\ell}} \cdot \pres^{\rm block}(x_{\ell + 1}\mid \context, x^{\ell})
\end{align*}
Note that,
\begin{align*}
& \quad \probof{\tau = \ell \mid X^{\ell}= x^{\ell}}  \\
   & = \probof{\tau \ge \ell \mid X^{\ell}  = x^{\ell}} - \sum_{x} \msmall(x \mid \context, x^{\ell}) \cdot \probof{\tau \ge \ell+1 \mid \context,  X^{\ell+1} = x^{\ell},x} \\
    & = \pab_{\ell}(x^{\ell}) -  \sum_{x} \msmall(x \mid \context,  x^{\ell}) \cdot \pab_{\ell+1}(x^{\ell}, x) \\
    & =  \pab_{\ell}(x^{\ell}) - \sum_{x} \min\{\pab_{\ell}(x^{\ell}) \mbig(x \mid \context,  x^{\ell}),  \msmall(x \mid \context,  x^{\ell})\}\\
    & = \sum_x \max\{\pab_{\ell}(x^{\ell}) \mbig(x \mid \context,  x^{\ell}) - \msmall(x\mid \context,  x^{\ell}), 0\}.
\end{align*}
\znew{And hence by the definition of $\pres^{\rm block}(x_{\ell + 1}\mid \context, x^{\ell})$ in \cref{eqn:pres-block}, we have}
\begin{align}
     & \quad \probof{O^{\ell+1} = x^{\ell+1}, \tau = \ell} \nonumber  \\
     & = \msmall(x^{\ell} \mid \context) \cdot \max\{\pab_{\ell}(x^{\ell}) \mbig(x_{\ell + 1} \mid  \context, x^{\ell}) - \msmall(x_{\ell + 1} \mid  \context, x^{\ell}), 0\}.\label{eqn:last_one_reject}
\end{align}
For the third term ($\tau < \ell$), by induction, and the generation process of $O^{\nt+1}$, we have
\begin{align}
    \probof{O^{\ell+1} = x^{\ell+1}, \tau < \ell} & =   \probof{O^{\ell} = x^{\ell}, \tau < \ell} \cdot  \probof{O_{\ell + 1} = x_{\ell + 1} \mid  O^{\ell} = x^{\ell},\tau < \ell } \nonumber \\
    & = \Paren{\probof{O^{\ell} = x^{\ell}} - \probof{O^{\ell} = x^{\ell}, \tau \ge \ell}} \cdot \mbig(x_{\ell + 1} \mid \context, x^{\ell})  \nonumber \\
    & = \Paren{\mbig(x^{\ell} \mid \context) - \msmall(x^{\ell} \mid \context) \pab_{\ell}(x^{\ell})} \cdot \mbig(x_{\ell + 1} \mid \context, x^{\ell}) \label{eqn:early_reject} 
\end{align}

Plugging \cref{eqn:all_accept,eqn:last_one_reject,eqn:early_reject} into \cref{eqn:break_down}, we get  $\forall x^{\ell + 1} \in \cX^{\ell + 1}$,
\[
\probof{O^{\ell+1} = x^{\ell + 1}}  = \mbig(x^{\ell + 1} \mid \context),
\]
completing the induction step and hence the proof of \cref{eqn:output_block_distribution} and \cref{thm:distribution_match}.

%% file: sections/proof_thm_optimal.tex
    We first state the following lemma, which when combined with \cref{lem:accept_prob}, shows that in one iteration, among all valid draft verification algorithms, \specblockt accepts each subsequence with the highest probability.

\begin{lemma}\label{lem:accept_prob_upper}
For draft verification algorithms that satisfy the constraints in \cref{lem:verify_goal}, we have $\forall i \le \nt$, and $x^i \in \cX^i$,
\[
    \probof{\tau \ge i \mid X^i = x^i} \le \pab_i (x^i).
\]
\end{lemma}

We defer the proof of the lemma to \cref{sec:proof_lem_prob_upper} and first prove \cref{thm:optimal_decoded_length} based on the lemma.

We start by breaking down the expected number of decoded tokens $\EE_{\draftverify}[N(i)]$ into the distribution of $N(i)$ on different sample paths. Let $O^* = O_1, O_2, \ldots, $ be the complete output sequence from speculative decoding. We set all tokens after \eos~to be \eos~as well. Then we have
\znew{
\begin{align*}
& \quad \EE_{\draftverify}[N(i)]  = \sum_{\ell = 1}^{\infty}  \probofsub{\draftverify}{N(i) \ge \ell}  = 
    \sum_{x^* \in \cX^*}  \sum_{\ell = 1}^{\infty}  \probofsub{\draftverify}{O^* = x^*, N(i) \ge \ell} 
\end{align*}
}
Hence it would be enough to prove the following.
\begin{lemma}\label{lem:length_dominate}
For all draft verification algorithms that satisfy the constraints in \cref{lem:verify_goal}, \znew{we have $\forall \context$, and output $x^{*} \in \cX^{*}$}
\begin{equation} \label{eqn:length_dominate}
     \probofsub{\draftverify}{\znew{O^{\*} = x^{*}}, N(i) \ge \ell \mid \context} \le \probofsub{\specblock}{\znew{O^{\*} = x^{*}}, N(i) \ge \ell \mid \context} 
\end{equation}
\end{lemma}

We prove the lemma by induction on the number of iterations $i$. \znew{We first prove the following lemma for all verification algorithms.
\begin{lemma} \label{lem:one_iteration_verify}
For all $\ell \le \nt$,
\[
    \probofsub{\draftverify}{O^ \ast=x^ \ast, \tau \geq \ell \mid \boldsymbol{c}}   = \probofsub{\draftverify}{O^{\ell}=x^{\ell}, \tau \geq \ell \mid \boldsymbol{c}} \cdot \mathcal{M}_b^*\left(x^{\ell+1:  \ast} \mid \boldsymbol{c}, x^{\ell}\right).
\]
\end{lemma}
}
\begin{proof}
\znew{
When $\probofsub{\draftverify}{O^{\ell}=x^{\ell}, \tau \geq \ell \mid \boldsymbol{c}} = 0$, the bound is trivial since both sides are 0. Otherwise we have 
\begin{align}
 & \probofsub{\draftverify}{O^ \ast=x^ \ast, \tau \geq \ell \mid \boldsymbol{c}}  \nonumber \\
= &\probofsub{\draftverify}{O^{\ell+1: \ast}=x^{\ell+1: \ast}, O^{\ell}=x^{\ell}, \tau \geq \ell \mid \boldsymbol{c}} \nonumber \\
= &\probofsub{\draftverify}{O^{\ell}=x^{\ell}, \tau \geq \ell \mid \boldsymbol{c}} \cdot  \probofsub{\draftverify}{O^{\ell+1: \ast}=x^{\ell+1: \ast} \mid O^{\ell}=x^{\ell}, \tau \geq \ell, \boldsymbol{c}} \nonumber
\end{align}
It would be enough to show that
\begin{equation}
    \probofsub{\draftverify}{O^{\ell+1: \ast}=x^{\ell+1: \ast} \mid O^{\ell}=x^{\ell}, \tau \geq \ell, \boldsymbol{c}} =  \mathcal{M}_b^*\left(x^{\ell+1:  \ast} \mid \boldsymbol{c}, x^{\ell}\right).
\label{eqn:past_independent}
\end{equation}
Note that 
\begin{align}
& \quad \mathcal{M}_b^*\left(x^{\ell+1:  \ast} \mid \boldsymbol{c}, x^{\ell}\right) \nonumber \\
     & =  \probofsub{\draftverify}{O^{\ell+1: \ast}=x^{\ell+1: \ast} \mid O^{\ell}=x^{\ell}, \boldsymbol{c}}  \nonumber\\
     & = \probofsub{\draftverify}{O^{\ell+1: \ast}=x^{\ell+1: \ast}, \tau \ge \ell \mid O^{\ell}=x^{\ell}, \boldsymbol{c}} + \probofsub{\draftverify}{O^{\ell+1: \ast}=x^{\ell+1: \ast}, \tau < \ell  \mid O^{\ell}=x^{\ell}, \boldsymbol{c}} \nonumber\\
     & = \probofsub{\draftverify}{\tau \ge \ell \mid O^{\ell}=x^{\ell}, \boldsymbol{c}} \probofsub{\draftverify}{O^{\ell+1: \ast}=x^{\ell+1: \ast} \mid O^{\ell}=x^{\ell}, \tau \ge \ell, \boldsymbol{c}}  \nonumber\\
     & \quad \quad \quad \quad  + \probofsub{\draftverify}{\tau < \ell \mid O^{\ell}=x^{\ell}, \boldsymbol{c}} \probofsub{\draftverify}{O^{\ell+1: \ast}=x^{\ell+1: \ast} \mid O^{\ell}=x^{\ell}, \tau < \ell, \boldsymbol{c}}  \label{eqn:decompose}
\end{align}
When $\probofsub{\draftverify}{\tau < \ell \mid O^{\ell}=x^{\ell}, \boldsymbol{c}} = 0$, we have $\probofsub{\draftverify}{\tau \ge \ell \mid O^{\ell}=x^{\ell}, \boldsymbol{c}} = 1$, and hence
\begin{align*}
    \probofsub{\draftverify}{O^{\ell+1: \ast}=x^{\ell+1: \ast} \mid O^{\ell}=x^{\ell}, \tau \geq \ell, \boldsymbol{c}} = \mathcal{M}_b^*\left(x^{\ell+1:  \ast} \mid \boldsymbol{c}, x^{\ell}\right).
\end{align*}
When $\probofsub{\draftverify}{\tau < \ell \mid O^{\ell}=x^{\ell}, \boldsymbol{c}} > 0$, since \draftverify~is a valid verification algorithm (\cref{def:draft_verify}), we have tokens starting from location $\ell+1$ is valid draw from $\mbig(\cdot \mid \boldsymbol{c}, x^{\ell})$, \ie
\[
        \probofsub{\draftverify}{O^{\ell+1: \ast}=x^{\ell+1: \ast} \mid O^{\ell}=x^{\ell}, \tau < \ell, \boldsymbol{c}} =  \mathcal{M}_b^*\left(x^{\ell+1:  \ast} \mid \boldsymbol{c}, x^{\ell}\right).
        \]
Plugging this into \cref{eqn:decompose} completes the proof. 
}
\end{proof}

\znew{
When $i = 1$, we have that $N(1) = \tau + 1$, where $\tau$ is the number of accepted tokens. 
Hence we have}
\znew{
\begin{align}
& \quad \probofsub{\draftverify}{O^* = x^*, N(1) \ge \ell \mid \context}  \nonumber \\
& =  \probofsub{\draftverify}{O^ \ast=x^ \ast, \tau \geq \ell - 1 \mid \boldsymbol{c}}  \nonumber \\
&  = \probofsub{\draftverify}{O^{\ell-1}=x^{\ell-1}, \tau \geq \ell - 1 \mid \boldsymbol{c}} \cdot \mathcal{M}_b^*\left(x^{\ell:  \ast} \mid \boldsymbol{c}, x^{\ell - 1}\right)  \nonumber \quad \quad \quad \quad \quad \quad \triangleleft \text{\cref{lem:one_iteration_verify}} \\
& = \msmall(x^{\ell-1} \mid \context) \mbig^{*}(x^{\ell:*} \mid \context, x^{\ell-1}) \probofsub{\draftverify}{ \tau \ge \ell - 1 \mid X^{\ell-1} = x^{\ell-1}, \context}, \nonumber
\end{align}
where the last equality is because $\probofsub{\draftverify}{O^{\ell-1}=x^{\ell-1}, \tau \geq \ell - 1 \mid \boldsymbol{c}}$ is the probability of the event that $O^{\ell-1}=x^{\ell-1}$
 is contained in the accepted tokens, which happens under the joint of two events: (1) The first $\ell-1$
 tokens in the draft block from the small model  $X^{\ell - 1} = x^{\ell-1}$. This probability is $ \msmall(x^{\ell-1} \mid \context)$; (2) Conditioned on $X^{\ell - 1} = x^{\ell-1}$, at least $\ell-1$
 tokens are accepted, this is $\probofsub{\draftverify}{ \tau \ge \ell - 1 \mid X^{\ell-1} = x^{\ell-1}, \context}$, and hence
 \[
    \probofsub{\draftverify}{O^{\ell-1}=x^{\ell-1}, \tau \geq \ell - 1 \mid \boldsymbol{c}} = \msmall(x^{\ell-1} \mid \context) \cdot \probofsub{\draftverify}{ \tau \ge \ell - 1 \mid X^{\ell-1} = x^{\ell-1}, \context}.
 \]
}
\znew{
Similarly, we have
\begin{align*}
    & \probofsub{\specblock}{O^* = x^*, N(1) \ge \ell \mid \context} \\ 
    = &  \msmall(x^{\ell-1} \mid \context) \mbig^{*}(x^{\ell:*} \mid \context, x^{\ell-1}) \probofsub{\specblock}{ \tau \ge \ell - 1 \mid X^{\ell-1} = x^{\ell-1}, \context}.
\end{align*}
Note that \cref{lem:accept_prob,lem:accept_prob_upper} imply that for all verification algorithm, we have
\[
    \probofsub{\draftverify}{ \tau \ge \ell - 1 \mid X^{\ell-1} = x^{\ell-1}, \context} \le \probofsub{\specblock}{ \tau \ge \ell - 1 \mid X^{\ell-1} = x^{\ell-1}, \context}.
\]
Combining these, we have
\begin{equation}
     \probofsub{\draftverify}{O^* = x^*, N(1) \ge \ell \mid \context}  \le \probofsub{\specblock}{O^* = x^*, N(1) \ge \ell \mid \context}. \label{eqn:one_iteration}
\end{equation}
}
Suppose the lemma holds for all iterations up to $i$, for the $(i+1)$th iteration, let $\tau_{i+1}$ be the number of tokens accepted in the $(i+1)$th iteration, we have
\begin{align}
   & \quad  \probofsub{\draftverify}{O^* = x^*, N(i+1) \ge \ell \mid \context} \nonumber \\
   & =  \sum_{\ell' < \ell} \probofsub{\draftverify}{O^* = x^*, N(i) = {\ell'}, N(i+1) \ge \ell \mid \context}   \nonumber \\
   & = \sum_{\ell' < \ell} \probofsub{\draftverify}{O^* = x^*, N(i) = {\ell'} \mid \context} 
    \probofsub{\draftverify}{ \tau_{i+1} \ge \ell - \ell' - 1 \mid \context, \znew{O^* = x^*,} N(i) = \ell'} \nonumber \\
    & = \probofsub{\draftverify}{O^* = x^* \mid \context}  \sum_{\ell' < \ell} \probofsub{\draftverify}{N(i) = {\ell'} \mid O^* = x^*, \context}  \nonumber \\
    & \qquad \qquad \qquad  \qquad  \cdot \probofsub{\draftverify}{\tau_{i+1} \ge \ell - \ell' - 1 \mid \context, \znew{O^* = x^*}, N(i) = \ell'}  \nonumber \\
    & = \mbig(x^* \mid \context)  \sum_{\ell' < \ell} \probofsub{\draftverify}{N(i) = {\ell'} \mid O^* = x^*, \context} \probofsub{\draftverify}{\tau_{i+1} \ge \ell - \ell' - 1 \mid \context, \znew{O^* = x^*}, N(i) = \ell'} 
    \label{eqn:prob_iteration_ipo}
\end{align}

Let $\eta_{\draftverify}$ be a random variable distributed according to $\probofsub{\draftverify}{N(i) = {\ell'} \mid O^* = x^*, \context}$,  and 
\[
    f_\draftverify(\eta) = \probofsub{\draftverify}{\tau_{i+1} \ge \ell - \eta - 1 \mid \context, \znew{O^* = x^*}, N(i) = \eta}.
\]
Plugging these into \cref{eqn:prob_iteration_ipo}, we have
\begin{align*}
    \probofsub{\draftverify}{O^* = x^*, N(i+1) \ge \ell \mid \context} =  \mbig(x^* \mid \context) \EE_{\eta_{\draftverify}} \left[ f_\draftverify(\eta_{\draftverify})\right]
\end{align*}

Note that let $\context_\eta = \context, x^{\eta}$, we have
\begin{align}
      f_\draftverify(\eta) & =  \probofsub{\draftverify}{\tau_{i+1} \ge \ell - \znew{\eta} - 1 \mid \context, \znew{O^* = x^*}, N(i) = \znew{\eta}} \nonumber \\
     &  =  \probofsub{\draftverify}{ \tau \ge \ell - \znew{\eta} - 1 \mid \context_\eta, O^* = x^{\eta+1: *}}  \label{eqn:conditioning}
\end{align}
where \cref{eqn:conditioning} is due to the iterative structure of speculative decoding and after generating $O^{\eta} = x^\eta$ in the first $i$ iterations ($N(i) = \znew{\eta}$), the next iteration is the same as generating from scratch with context $\context_\eta = \context, x^{\eta}$.
Similarly, we have
\[
    f_\specblock(\eta) = \probofsub{\specblock}{\tau \ge \ell - \eta - 1 \mid \context_\eta, O^* = x^{\eta+1: *} }.
\]
Note that $\forall \context$,  $x^* \in \cX^*$, and $i$, we have
\begin{align} 
    \probofsub{\specblock}{\tau \ge i  \mid \context, O^* = x^{*} }  & = 
    \frac{\probofsub{\specblock}{O^* = x^{*}, \tau \ge i \mid \context  } }{\probofsub{\specblock}{ O^* = x^{*} \mid \context  }} \nonumber \\
    & =  \frac{\probofsub{\specblock}{O^* = x^{*}, \tau \ge i \mid \context  } }{\mbig^*(x^{*} \mid \context  )} \nonumber \\
    & \ge \frac{\probofsub{\draftverify}{O^* = x^{*}, \tau \ge i \mid \context  } }{\mbig^*(x^{*} \mid \context  )} 
    \label{eqn:bvtodv}  \\
    & = \probofsub{\draftverify}{\tau \ge i  \mid \context, O^* = x^{*} },
    \label{eqn:conditional_larger}
\end{align}
where \cref{eqn:bvtodv} is due to \cref{eqn:one_iteration} and that $N(1) = \tau + 1$.
\cref{eqn:conditional_larger} implies that $f_\specblock(\eta) \ge f_\draftverify(\eta)$, and hence we have 
\begin{align*}
      \probofsub{\draftverify}{O^* = x^*, N(i+1) \ge \ell \mid \context} & =  \znew{\mbig(x^* \mid \context)} \EE_{\eta_{\draftverify}} \left[ f_\draftverify(\eta)\right] \\
      & \le  \znew{\mbig(x^* \mid \context)} \EE_{\eta_{\draftverify}} \left[ f_\specblock(\eta)\right].
\end{align*}

Note that $\probofsub{\specblock}{O^* = x^*, N(i+1) \ge \ell \mid \context} = \znew{\mbig(x^* \mid \context)} \EE_{\eta_{\specblock}} \left[ f_\specblock(\eta)\right]$. It would be enough to prove that
\begin{equation}\label{eqn:expectation_greater}
    \EE_{\eta_{\draftverify}} \left[ f_\specblock(\eta)\right] \le \EE_{\eta_{\specblock}} \left[ f_\specblock(\eta)\right].
\end{equation}

\znew{Next we prove \cref{eqn:expectation_greater} using the lemma below.
\begin{lemma}[\cite{quirk1962admissibility}] \label{lem:stochastic_dominance}
 Let $f: \RR \rightarrow \RR$ be an increasing function and $X_1$ stochastically dominates $X_2$, meaning $\forall x,$ we have $\probof{X_1 \ge x} \ge \probof{X_2 \ge x}$, then we have
 \[
    \EE[f(X_1)] \ge \EE[f(X_2)].
 \]
 \end{lemma}
}
By the induction hypothesis, we have $\eta_{\specblock}$ stochastically dominates \citep{quirk1962admissibility}  $\eta_{\draftverify}$ for any valid verification algorithm. It remains to show that $f_\specblock(\eta)$ is an increasing function. By definition, since $0 \le \tau \le \gamma$, when $\eta < \ell - \nt - 1$, $f_\specblock(\eta) = 0$  and when $\eta > \ell - 1$, $f_\specblock(\eta) = 1$.
When $\ell - \nt - 1 \le \eta \le \ell - 2$, by definition and \cref{lem:accept_prob}, $f_\specblock(\eta)  = \pab_{\ell-\eta-1}(x^{\eta + 1:\ell-1} \mid \context, x^\eta)$. To see that $\pab_{\ell-\eta-1}(x^{\eta + 1:\ell-1} \mid \context, x^\eta)$ is an increasing function of $\eta$ , for $\eta' = \eta + 1$, we can obtain $\pab_{\ell-\eta'-1}(x^{\eta' + 1:\ell-1} \mid \context, x^{\eta'})$ by following the same recursion steps as in \cref{eqn:def_as} but replacing $\pab_{1}(x^{\eta + 1:\ell-1} \mid \context, x^{\eta})$ with $\pab_{0}(x^{\eta + 2:\ell-1} \mid \context, x^{\eta+1}) = 1$, and hence only increasing the values. 
  This proves that $f_\specblock$ is increasing and hence \cref{eqn:expectation_greater} holds. This implies that the induction step holds due to \cref{lem:stochastic_dominance}, completing proof of \cref{thm:optimal_decoded_length}.

%% file: sections/proof_lem_acc_prob.tex
Note that in Line 4 of \cref{alg:block_verify}, $\pab_i = \pab_i(X^i)$. 
We prove the statement by backward induction. When $i = \nt$, we have by definition of $\tb_\nt$ in \cref{fig:missing_definition}, $\forall x^\nt \in \cX^\nt$,
\[
    \probof{\tau \ge \nt \mid X^\nt = x^\nt} =\tb_\nt = \pab_\nt(x^\nt).
\]
Suppose the statement holds for $i \ge \ell$. When $i = \ell - 1$, we have
\begin{align}
     & \;\;\;\;\; \probof{\tau \ge \ell - 1 \mid X^{\ell - 1} = x^{\ell - 1}} \nonumber \\ 
     & = \sum_{x_{\ell} \in \cX} \msmall(x_{\ell} \mid 
    \context, x^{\ell - 1}) \cdot \probof{\tau \ge \ell - 1 \mid X^{\ell} = x^{\ell}} \nonumber  \\ 
    & = \sum_{x_{\ell} \in \cX} \msmall(x_{\ell} \mid 
    \context, x^{\ell - 1}) \cdot \Paren{\probof{\tau \ge \ell \mid X^{\ell} = x^{\ell}} + \probof{\tau = \ell - 1 \mid X^{\ell} = x^{\ell}}}  \nonumber\\
          &  = \sum_{x_{\ell} \in \cX} \msmall(x_{\ell} \mid  \context,  x^{\ell - 1}) \cdot \Paren{\probof{\tau \ge \ell\mid X^{\ell} = x^{\ell}}  + \probof{\tau < \ell \mid X^{\ell} = x^{\ell}} \cdot \znew{\tb_{\ell-1}}} \label{eqn:l2lm1}  \\
          & =  \sum_{x_{\ell} \in \cX} \msmall(x_{\ell} \mid  \context,  x^{\ell - 1}) \cdot \Paren{\pab_{\ell}(x^\ell) + (1 - \pab_{\ell}(x^\ell)) \cdot \znew{\tb_{\ell-1}}}, \nonumber  \\
          & = \sum_{x_{\ell} \in \cX} \msmall(x_{\ell}\mid  \context,  x^{\ell - 1}) \cdot \pab_{\ell}(x^\ell) + \znew{\tb_{\ell-1}} \cdot \sum_{x_{\ell} \in \cX} \msmall(x_{\ell} \mid \context, x^{\ell - 1}) \cdot (1 - \pab_{\ell} (x^\ell)) \nonumber \\
          & = \sum_{x_{\ell} \in \cX} \msmall(x_{\ell}\mid  \context,  x^{\ell - 1}) \cdot \pab_{\ell}(x^\ell) + \znew{\tb_{\ell-1}} \cdot (1  - \sum_{x_{\ell} \in \cX} \msmall(x_{\ell} \mid \context, x^{\ell - 1}) \pab_{\ell} (x^\ell)). \label{eqn:induction_sum}
\end{align}
\znew{\cref{eqn:l2lm1} above holds since $\tau = \ell -1$ happens under the joint event of $\tau < \ell$ and $\eta_{\ell-1} < \tb_{\ell-1}$.}
Note that in the definition of $\znew{\tb_{\ell-1}}$ (\cref{eqn:pacc_block}),
\begin{align*}
    & \;\;\;\; \sum_x\max\{\pab_{\ell-1}(x^{\ell-1})\mbig(x \mid \context, x^{\ell-1}) - \msmall(x \mid \context, x^{\ell-1}), 0\} \\
    & = \sum_x \Paren{ \pab_{\ell-1}(x^{\ell-1})\mbig(x \mid \context,  x^{\ell-1})  - \min\{ \pab_{\ell-1}(x^{\ell-1})\mbig(x \mid \context, x^{\ell-1}), \msmall(x \mid \context, x^{\ell-1})\}} \\
    & =  \pab_{\ell-1}(x^{\ell-1}) - \sum_x  \min\{ \pab_{\ell-1}(x^{\ell-1})\mbig(x \mid \context, x^{\ell-1}), \msmall(x \mid \context, x^{\ell-1})\}
\end{align*}
\znew{Plugging this into \cref{eqn:pacc_block},}
\begin{align}
   \znew{\tb_{\ell-1}} = \frac{  \pab_{\ell-1}(x^{\ell-1}) - \sum_x  \min\{ \pab_{\ell-1}(x^{\ell-1})\mbig(x \mid \context, x^{\ell-1}), \msmall(x \mid \context, x^{\ell-1})\}}{1 - \sum_x \min\{ \pab_{\ell-1}(x^{\ell-1})\mbig(x \mid \context, x^{\ell-1}), \msmall(x \mid \context, x^{\ell-1})\}}. \label{eqn:prob_not_bot}
\end{align}

Moreover, we have by the definition of $ \pab_{\ell}(x^\ell)$,
\begin{align}
    \sum_{x_{\ell} \in \cX} \msmall(x_{\ell} \mid \context, x^{\ell - 1}) \cdot  \pab_{\ell} (x^\ell)
    & = \sum_{x_{\ell} \in \cX}  \min\{ \pab_{\ell-1}(x^{\ell-1})\mbig(x_{\ell}  \mid \context, x^{\ell-1}), \msmall(x_{\ell}  \mid \context, x^{\ell-1})\} \nonumber \\
    & = \sum_{x\in \cX}  \min\{ \pab_{\ell-1}(x^{\ell-1})\mbig(x  \mid \context, x^{\ell-1}), \msmall(x \mid \context, x^{\ell-1})\}. \label{eqn:sum_prod}
\end{align}

Plugging \cref{eqn:sum_prod} and \cref{eqn:prob_not_bot} into \cref{eqn:induction_sum}, we get
\[
    \probof{\tau \ge \ell - 1 \mid X^{\ell - 1} = x^{\ell - 1}} = \pab_{\ell-1}(x^{\ell-1}),
\]
as desired. The lemma hence follows by induction.

%% file: sections/proof_lem_prob_upper.tex
Recall that we use $O^{\nt+1}$ to denote the sequence $(X^\tau, Y, Z^{\nt-\tau})$ in \cref{eqn:verify_goal}. \znew{Without loss of generality, we only consider $o^\ell$ such that $\probof{O^\ell = o^\ell} > 0$ and $\probof{X^\ell = o^\ell} > 0$ since otherwise  $\probof{\tau \ge i \mid X^i = x^i}$ is either zero or ill-defined. We break the proof into the two cases below.}

If $\forall i < \ell$, it satisfies that
$
    \pab_{i-1} (x^{i-1}) \mbig(x_{i} \mid \context, x^{i-1}) \le \msmall(x_{i} \mid \context, x^{i-1}),$
then we have in the recursive formula of $\pab_i$'s in \cref{alg:block_verify}, we always have
\[
    \pab_{i} (x^{i}) =  \pab_{i-1} (x^{i-1}) \frac{\mbig(x_{i} \mid \context, x^{i-1})}{\msmall(x_{i} \mid \context, x^{i-1})},
\]
and hence 
\[
  \pab_{\ell-1} (x^{\ell-1}) =  \frac{\mbig(x^{\ell-1} \mid \context)}{\msmall(x^{\ell-1} \mid \context)},
\]
And for $x^\ell$, we have 
\[
 \pab_{\ell} (x^{\ell}) = \min\{ \frac{\mbig(x^{\ell} \mid \context)}{\msmall(x^{\ell} \mid \context)}, 1\}.
\]
Note that %
\begin{align*}
        \probof{O^\ell = x^\ell, \tau \ge \ell}  %
         =  \probof{X^\ell = x^\ell} \probof{\tau \ge \ell \mid X^\ell = x^\ell}  = \msmall(x^{\ell} \mid \context) \probof{\tau \ge \ell \mid X^\ell = x^\ell} 
\end{align*}
\znew{Moreover, we have
\[
\probof{O^\ell = x^\ell, \tau \ge \ell} \le  \probof{O^\ell = x^\ell} = \mbig(x^{\ell} \mid \context),
\]
and 
\[
    \probof{O^\ell = x^\ell, \tau \ge \ell} =  \probof{X^\ell = x^\ell} \probof{\tau \ge \ell \mid X^\ell = x^\ell} \le  \probof{X^\ell = x^\ell} = \msmall(x^{\ell} \mid \context).
\]
Hence
\[
    \probof{\tau \ge \ell \mid X^\ell = x^\ell}  = \frac{\probof{O^\ell = x^\ell, \tau \ge \ell}}{\msmall(x^{\ell} \mid \context)} \le \min\{\frac{\mbig(x^{\ell} \mid \context)}{\msmall(x^{\ell} \mid \context)}, 1\}  = \pab_{\ell} (x^{\ell}).
\]}

In the other case, \znew{there must exist some $i$ such that $
    \pab_{i-1} (x^{i-1}) \mbig(x_{i} \mid \context, x^{i-1}) > \msmall(x_{i} \mid \context, x^{i-1}),$ then we have
    \[
        \pab_i(x^i) = \min \{\frac{\pab_{i-1} (x^{i-1}) \mbig(x_{i} \mid \context, x^{i-1})}{\msmall(x_{i} \mid \context, x^{i-1})}, 1 \} = 1.
    \]}

WLOG, let $i$ be the largest such index. \znew{In this case, we have $\forall i < j < \ell,
    \pab_{j-1} (x^{j-1}) \mbig(x_{j} \mid \context, x^{j-1}) \le  \msmall(x_{j} \mid \context, x^{j-1}),$ and hence}
\[
    \pab_\ell(x^\ell) = \pab_i(x^i) \frac{\mbig^{\ell-i}(x^{i+1:\ell } \mid \context, x^i)}{\msmall^{\ell-i}(x^{i+1:\ell } \mid \context, x^i)} = \frac{\mbig^{\ell-i}(x^{i+1:\ell } \mid \context, x^i)}{\msmall^{\ell-i}(x^{i+1:\ell } \mid \context, x^i)}.
\]
Moreover, by definition, we have
\znew{\[
\pab_{i-1} (x^{i-1}) \le \pab_{i-2} (x^{i-2})\frac{\mbig(x_{i-1} \mid \context, x^{i-2})}{\msmall(x_{i-1} \mid \context, x^{i-2})} \le \ldots \le \frac{\mbig^{i-1}(x^{i-1 } \mid \context)}{\msmall^{i-1}(x^{i-1 } \mid \context)},
\]
and hence when $
    \pab_{i-1} (x^{i-1}) \mbig(x_{i} \mid \context, x^{i-1}) > \msmall(x_{i} \mid \context, x^{i-1}),$
\[
    \frac{\mbig(x^i \mid \context)}{\msmall(x^i \mid \context)} =  \frac{\mbig^{i-1}(x^{i-1 } \mid \context)}{\msmall^{i-1}(x^{i-1 } \mid \context)} \cdot  \frac{\mbig(x_i \mid \context, x^{i-1 })}{\msmall(x_i \mid \context, x^{i-1})} \ge \pab_{i-1} (x^{i-1}) \cdot  \frac{\mbig(x_i \mid \context, x^{i-1 })}{\msmall(x_i \mid \context, x^{i-1})} > 1.
\]
}
Hence
\[
  \probof{O^i = x^i, \tau \ge i} =    \probof{X^i = x^i}  \probof{\tau \ge i \mid X^i = x^i}  \le \msmall(x^i  \mid \context) < \mbig(x^i  \mid \context),
\]
and 
\[
\probof{O^i = x^i, \tau < i} = \probof{O^i = x^i} - \probof{O^i = x^i, \tau \ge i} = \mbig(x^i  \mid \context) - \msmall(x^i  \mid \context )  > 0.
\]

Note that when $O^i = x^i, \tau < i$, by constraints in \cref{eqn:verify_goal}, we have
\[
    \probof{O^{i+1:\ell} = x^{i+1:\ell} \mid O^i = x^i, \tau < i} = \mbig^{\ell-i}(x^{i+1:\ell} \mid \context, x^i).
\]
This implies
\begin{align*}
  \probof{O^{\ell} = x^\ell} 
  & = \probof{O^i = x^i, \tau < i} \cdot \probof{O^{i+1:\ell} = x^{i+1:\ell} \mid O^i = x^i, \tau <  i} \\ & \quad \quad \quad \quad + \probof{O^i = x^i, \tau \ge  i} \probof{O^{i+1:\ell} = x^{i+1:\ell} \mid O^i = x^i, \tau \ge  i}\\
  & = \probof{O^i = x^i, \tau < i} \cdot \mbig^{\ell-i}(x^{i+1:\ell} \mid \context, x^i)  \\ 
 & \quad \quad \quad \quad + \probof{O^i = x^i, \tau \ge  i} \probof{O^{i+1:\ell} = x^{i+1:\ell} \mid O^i = x^i, \tau \ge  i}
\end{align*}
Moreover, we have
\[
     \probof{O^{\ell} = x^\ell} = \mbig(x^i) \mbig^{\ell-i}(x^{i+1:\ell} \mid \context, x^i) 
\]
Combining both, we get 
\begin{align*}
    1 & = \frac{\probof{O^{\ell} = x^\ell} }{\mbig(x^i) \mbig^{\ell-i}(x^{i+1:\ell} \mid \context, x^i)} \\
    & = \probof{\tau < i \mid O^i = x^i}  + \probof{\tau \ge i \mid O^i = x^i}  \frac{\probof{O^{i+1:\ell} = x^{i+1:\ell} \mid O^i = x^i, \tau \ge  i}}{\mbig^{\ell-i}(x^{i+1:\ell} \mid \context, x^i)}, \\
    & = 1 - \probof{\tau \ge i \mid O^i = x^i} \Paren{  \frac{\probof{O^{i+1:\ell} = x^{i+1:\ell} \mid O^i = x^i, \tau \ge  i}}{\mbig^{\ell-i}(x^{i+1:\ell} \mid \context, x^i)}-1},
\end{align*}
and this implies that \znew{(note by assumption $\probof{\tau \ge i \mid O^i = x^i} \neq 0$),}
\begin{equation} \label{eqn:dist_accepted_path}
    \probof{O^{i+1:\ell} = x^{i+1:\ell} \mid O^i = x^i, \tau \ge  i} = \mbig^{\ell-i}(x^{i+1:\ell} \mid \context, x^i). 
\end{equation}
Hence 
\begin{align*}
      \probof{O^\ell = x^\ell, \tau \ge \ell } & \le  \probof{O^\ell = x^\ell, \tau \ge i }  \\
     & = \probof{O^i = x^i, \tau \ge  i} \probof{O^{i+1:\ell} = x^{i+1:\ell} \mid O^i = x^i, \tau \ge  i} \\
     & \le \msmall(x^i\mid \context) \pab_i(x^i) \mbig^{\ell-i}(x^{i+1:\ell} \mid \context, x^i) \\
     & = \msmall(x^i\mid \context) \mbig^{\ell-i}(x^{i+1:\ell} \mid \context, x^i).
\end{align*}
\znew{If $\probof{\tau \ge \ell \mid X^\ell = x^\ell} > \pab_\ell(x^\ell)$, we have}
\begin{align*}
      \probof{O^\ell = x^\ell, \tau \ge \ell } & = \probof{X^\ell = x^\ell} \probof{\tau \ge \ell \mid X^\ell = x^\ell} \\
      & > \msmall(x^\ell \mid \context) \pab_\ell(x^\ell) \\
      & =\msmall(x^\ell \mid \context) \frac{\mbig^{\ell-i}(x^{i+1:\ell } \mid \context, x^i)}{\msmall^{\ell-i}(x^{i+1:\ell } \mid \context, x^i)} \\
      & = \msmall(x^i\mid \context) \mbig^{\ell-i}(x^{i+1:\ell} \mid \context, x^i),
\end{align*}
which leads to a contradiction. This completes the proof.

%% file: sections/greedy.tex
\section{Greedy Block Verification} \label{sec:block_verify_greedy}

In this section, we show that it is possible to accept more tokens than 
block verification (\cref{alg:block_verify}) \emph{in one iteration}
 with a modification to the speculative decoding framework in 
\cref{alg:speculative_decoding_framework} that allows the decoding 
logic to 
depend on the previous accept/reject decisions. 
 \znew{However, as shown in \cref{tab:exp_comparison_greedy}, the resulting algorithm, greedy block verification,  doesn't improve over block verification. We include the description and analysis of the algorithm as a theoretical result. The claims in the main paper holds independent of the results in this section.}

We start by introducing the algorithm 
(\cref{alg:block_verify_greedy}) and then discuss the necessary 
modifications to maintain the identity distribution 
guarantee.
\newcommand{\tpab}{\tilde{\pab}}
\begin{algorithm}[h]
	\caption{Greedy block verification}
	\label{alg:block_verify_greedy}
	\begin{algorithmic}[1]
		\REQUIRE{Draft block $X^{\nt}$; small model distributions $\forall i< 
		\nt, \msmall(\cdot \mid \context, X^{i})$; target model distributions 
		$\forall i \le \nt, \mbig(\cdot \mid \context, X^{i})$.}
		\STATE Sample $\eta_1, \ldots, \eta_{\nt} \sim U(0, 1)$.
		\STATE Set $\tau = 0$, $\pab_0 = 1$.
		\FOR{$i = 1, \ldots,$ \diff{$\nt-1$}}
		\STATE \diff{Set $\tpab_i = \tpab_{i-1}\frac{\mbig(X_{i} \mid 
		\context,  X^{i-1})}{\msmall(X_{i} \mid \context, X^{i-1})}$.}
		\STATE Set \diff{$h_i = \frac{\sum_x\max\{\tpab_i \mbig(x \mid 
		\context, X^i) 
		- 
		\msmall(x \mid \context, X^i), \, 0\}}{\sum_x\max\{\msmall(x \mid 
		\context, X^i) - 
		\tpab_i  \mbig(x \mid \context, X^i), \, 
		0\}}$}
		\IF{$\eta_i \le h_i$}
		\STATE Set $\tau = i$.
		\ELSE
		\STATE \textbf{continue.} %
		\ENDIF
		\ENDFOR
		\STATE \diff{$\tpab_\nt = \tpab_{\nt-1}\frac{\mbig(X_{\nt} \mid 
			\context,  X^{\nt-1})}{\msmall(X_{\nt} \mid \context, X^{\nt-1})}$}
		\IF{\diff{$\eta_{\nt} < \tpab_\nt$}}
		\STATE Set $\tau = \nt$, and sample $Y$ from $\mbig(\cdot \mid 
		\context, X^{\nt})$.
		\ELSE
		\STATE Sample $Y$ from \diff{$\pres^{\rm greedy}(\cdot \mid 
		\context, 
		X^{\tau})$} as below:
		\begin{equation}
		\pres^{\rm greedy}(x \mid \context, X^{i}) = \frac{\max\{ 
			\diff{$\tpab_{i}$} \cdot \mbig(x \mid \context, X^{i}) - \msmall(x 
			\mid \context, X^{i}), 0 \}}{\sum_{x'\in\cX} \max\{ \diff{$\tpab_{i}$} 
			\cdot \mbig(x' \mid \context, X^{i}) - \msmall(x' \mid \context, 
			X^{i}), 0 \}}. \label{eqn:pres_greedy}
		\end{equation}
		\ENDIF
		\STATE \textbf{Return} $X^\tau, Y$.
	\end{algorithmic}
\end{algorithm}

The above greedy block verification algorithm has a similar procedure as 
block verification (\cref{alg:block_verify}) with differences in the setting of 
of acceptance probabilities and residual distributions, as highlighted.

Similar to \cref{alg:block_verify}, \cref{alg:block_verify_greedy} maintains a 
list of probabilities $\tpab_i$'s, which satisfies that $\min\{1, \tpab_i\}$ is 
the probability that the subblock $X^i$ is accepted. $h_i$'s are chosen to achieve the above acceptance 
guarantee, and $\pres^{\rm greedy}$'s are chosen to maintain the 
identical 
distribution guarantee. 

Note that compared to $\pab_i$ in block verification, the recursive 
definition of 
$\tpab_i$ doesn't have a 
minimum over one term, hence it 
is always an upper bound on $\pab_i$'s. This leads to a higher acceptance 
probability for every subblock in greedy block verification 
(\cref{thm:greedy_one_iter}). However, 
\cref{alg:block_verify_greedy} cannot be used directly in the iterative 
implementation of speculative decoding in 
\cref{alg:speculative_decoding_framework}. To see this, consider the 
simple example in \cref{sec:example}. Greedy block verification will 
perform the following:

Accept $X_1X_2 = \a\b, \b\a, \b\b$ with probability one, and sample an 
extra token from $\mbig(\cdot).$ Accept 
$X_1X_2 = \a\a$ 
with 
probability 1/4 and sample an extra token from $\mbig(\cdot).$ When  
$X_1X_2 = \a\a$  is rejected,  accept no tokens and sample a correction 
token $Y = \b$. Note that in this case, if the algorithm uses $Y$ as the 
context for the next iteration and sample based on 
$\mbig$, the next token 
will be $\a$ with probability $1/3$. This makes the total probability of 
generating $\b\a$ as the first two tokens
\begin{align*}
\msmall(\b\a) \probof{\text{Accept } \b\a} + \msmall(\a\a)  
\probof{\text{Reject } \a\a \text{ and } Y= \b}  \mbig(\a) & = 2/9 
\cdot 
1 + 4/9 \cdot 3/4 \cdot 1/3
\\
& = 1/3,
\end{align*}
which is higher than $\mbig(\b \a) = 2/9$. This violates the  identical 
distribution guarantee. Below we introduce a distribution 
modification algorithm, which can be used with 
\cref{alg:block_verify_greedy} to maintain the identical distribution 
guarantee.

\newcommand{\mnew}{\mathcal{M}_{\rm new}}
\begin{algorithm}
	\caption{Distribution modification}
	\label{alg:distribution_modify_greedy}
	\begin{algorithmic}[1]
		\REQUIRE{Small model $\msmall$; target model $\mbig$; draft length 
		$\gamma$; generated tokens from 
		\cref{alg:block_verify_greedy} $X^\tau, Y$.}
	\STATE Let $\mbig'$ be such that $\forall  i \le \gamma - \tau - 1$, and 
	$x^i \in \cX^i$, $\mnew(x_i \mid \context, X^\tau, Y, x^{i-1}) =$
	\begin{equation} \label{eqn:model_new}
		 \frac{\max\{ 
			\mbig(\context, X^\tau, Y, x^{i} ) - \msmall(\context, X^\tau, Y, 
			x^{i} ), 0 \}}{\sum_{x'\in\cX} \max\{ 
			\mbig(\context, X^\tau, Y, x^{i-1}, x' ) - \msmall(\context, X^\tau, 
			Y, 
			x^{i-1}, x' ), 0 \}},
	\end{equation}
\hfill \coloredcomment{Modify the distribution at rejected locations.} \\
	and $\forall i> \gamma - \tau - 1$, and 
	$x^i \in \cX^i$,	
	\[
		\mnew(x_i \mid \context, X^\tau, Y, x^{i-1})  = \mbig(x_i \mid 
		\context, X^\tau, Y, x^{i-1}) 
	\]
	\hfill \coloredcomment{Keep the distributions for future locations 
	unchanged.} 
	\STATE \textbf{Return} $\mnew$.
	\end{algorithmic}
\end{algorithm}

It can be shown that if $X^\tau, Y$ are returned in 
\cref{alg:block_verify_greedy}, and the next $\tau - 
\nt - 1$ tokens are sampled 
according to $\mnew$ from \cref{alg:distribution_modify_greedy}, the 
identical distribution 
guarantee is maintained. In particular, we have the following lemma:
\begin{lemma} \label{lem:greedy_block_identical}
	Let $X^\gamma \sim \msmall^\gamma(\cdot \mid \context)$ be the 
	draft tokens and $X^\tau, Y$ be the output from  
	\cref{alg:block_verify_greedy}. Let $\mnew$ be the modified distribution 
	based on \cref{alg:distribution_modify_greedy}, and 
	$Z^{\gamma-\tau-1} \sim \mnew^{\gamma-\tau-1}(\cdot \mid 
	\context, X^\tau, Y)$. Then we have
	\[
		X^\tau, Y, Z^{\gamma - \tau-1} \sim \mbig^\gamma(\cdot \mid 
		\context).
	\]
\end{lemma}

The proof is presented in \cref{proof:greedy_optimal}. The above leads to 
the following speculative decoding algorithm with 
greedy block verification, presented in 
\cref{alg:speculative_decoding_framework_greedy}. Note that 
\cref{lem:greedy_block_identical} implies that it maintains the identical distribution guarantee.

\begin{algorithm}[h]
	\caption{Speculative decoding with greedy block verification}
	\label{alg:speculative_decoding_framework_greedy}
	\begin{algorithmic}[1]
		\REQUIRE{Prefix $\context$, large model $\mbig$, draft model 
		$\msmall$. Draft length $\nt$.. }
		\WHILE{$\eos \notin (X^\tau, Y)$}
		\STATE Sample $X_1, \ldots, X_{\nt} \sim  \msmall(\cdot \mid 
		\context)$ using autoregressive sampling, keep the conditional\\
		probabilities at each step $\msmall(\cdot \mid \context, X^{i})$ for $i 
		= 0, \ldots, {\nt-1}$. \hfill \coloredcomment{Obtain draft block.}
		\STATE Call the large model $\mbig$ and compute conditional 
		probabilities $\mbig(\cdot \mid \context, X^{i})$\footnotemark\\
		for $i = 0, 1, \ldots, \nt$ in parallel. \hfill \coloredcomment{Parallel 
		scoring.}
		\STATE Get the accepted tokens with draft verification \hfill 
		\coloredcomment{Draft verification and correction.}
		\[
		X^\tau, Y = \draftverify(X^{\nt}, \{\msmall(\cdot \mid \context, X^{i}) 
		\}_{i = 0}^{\nt-1}, \{\mbig(\cdot \mid \context, X^{i}) \}_{i = 0}^{\nt}).
		\]
		\STATE $\context \leftarrow \context, X^\tau, Y.$ \hfill
		\coloredcomment{Add decoded tokens to the prefix.}
		\STATE \diff{$\mbig \leftarrow \textsc{DistributionModify}(\mbig, 
		\msmall, \gamma, X^\tau, Y)$} \hfill
		\coloredcomment{Modify target distribution.}
		\ENDWHILE
	\end{algorithmic}
\end{algorithm}

\subsection{Comparison to block verification.} In one draft iteration, with 
the same pair of draft and target distributions, greedy block verification is 
always better.
\begin{theorem}[Informal] \label{thm:greedy_one_iter}
	In one draft 
	iteration with the same models $\msmall, \mbig$ and draft length 
	$\gamma$,  greedy block verification 
	decodes at least as many tokens as block verification.
\end{theorem}

The theorem is proved in \cref{proof:greedy_optimal}. However, due to the 
distribution modification step, the target distribution
might change after the first iteration, which might affect the expected 
number of accepted tokens. For example, in the Bernoulli example
considered in \cref{sec:example}, when the draft block $X_1X_2=\a\a$ and 
they 
are rejected by greedy block verification. It can be shown that the modified 
distribution will be a point mass on token $\b$. And in future iterations, 
if the algorithm still uses $\msmall$ as the draft model, there is lower 
chance that the draft tokens will be accepted. Hence, theoretically it is 
unclear 
whether one approach dominates the other.

\footnotetext{{In cases 
		where $\mbig$ is not the original large transformer model. $\mbig$ 
		can be obtained by evaluating using the original large model, and then 
		perform the modification in \cref{eqn:model_new}.}}
\textbf{Empirical comparison.} We conduct the same set of experiments in 
\cref{sec:experiments} on 
greedy 
block verification to compare the two approaches empirically. We list the 
block efficiency comparison when PALM-2-XXS is used as the drafter and 
$\nt = 8$ in \cref{tab:exp_comparison_greedy}. As we can see, while 
greedy block verification still consistently improves over token verification, 
the improvement is less significant compared to block verification. The 
trend is the same for wall clock numbers as well as in other 
parameter settings. Hence 
we recommend using block verification instead of the greedy 
version. 

\begin{table*}[t]
	\caption{Block efficiency comparison between among token verification, 
	block verification, and greedy block verification with $\nt = 8$. Each 
	statistic is computed 
	using $1000$ test prompts from different datasets on various tasks 
	(each run is an average with 3 different random seeds).}
	\setlength{\tabcolsep}{2pt}
	\begin{center}
		\begin{tabular}{  c c c c }
			\toprule
			Dataset 
			&  Token Verification & Block verification & Greedy block 
			verification   \\ 
			\midrule
LM1B & $3.21$ & $\mathbf{3.49}$ & $3.30$\\
GPT Prompt & $3.41$ & $\mathbf{3.76}$ & $3.51$ \\
WebQA & $3.44$ & $\mathbf{3.70}$ & $3.52$\\
PIQA & $3.40$ & $\mathbf{3.68}$ & $3.49$  \\
ShareGPT & $3.34$ & $\mathbf{3.62}$ &$3.44$ \\
XSum & $3.49$ & $\mathbf{3.76}$ & $3.59$ \\
GSM8K & $3.81$ & $\mathbf{4.15}$  & $3.96$ \\
WMT-DeEn & $3.19$ & $\mathbf{3.41}$ &   $3.26$\\
			\bottomrule
		\end{tabular}
	\end{center}
	\label{tab:exp_comparison_greedy}
\end{table*}

\subsection{Proof of 
\cref{lem:greedy_block_identical}}\label{proof:greedy_dist_mathc}
In the proof, we ignore the context $\context$ and the proof will generalize 
to arbitrary $\context$. We start by introducing two useful quantities.

\begin{align}
	\premain(x^{i}) & \eqdef \sum_x \max\{\cM_b(x^{i}, x) - \cM_s(x^{i}, 
	x), 0\}, \label{eqn:premain_block}\\
	\prej(x^{i}) & \eqdef \sum_x \max\{\cM_s(x^{i}, x) - \cM_b(x^{i}, x), 0\}
	\label{eqn:prej_block}.
\end{align}

Note that $\pab_i$ in \cref{alg:block_verify_greedy} depends on the draft 
block $x^i$, and by the recursive definition of $\tpab_i$'s, we have
$
\tpab_i = \frac{\mbig(x^i)}{\msmall(x^i)}.$ Hence we have
\[
	h_i = \frac{\sum_x\max\{\tpab_i \mbig(x \mid 
		\context, X^i) 
		- 
		\msmall(x \mid \context, X^i), \, 0\}}{\sum_x\max\{\msmall(x \mid 
		\context, X^i) - 
		\tpab_i  \mbig(x \mid \context, X^i), \, 
		0\}} = \frac{	\premain(x^{i}) }{	\prej(x^{i}) }.
\]
Moreover, the expression for $\tpab_i$ also implies that $\pres^{\rm greedy}(\cdot \mid \context, X^\tau) = \mnew(\cdot \mid \context, X^\tau)$ (defined in \cref{eqn:pres_greedy} and \cref{eqn:model_new} resepectively).

	We now prove the following lemma about the acceptance length 
	$\tau$ in \cref{alg:block_verify_greedy}.
	\begin{lemma}\label{lem:block-accept}
		For all $\ell \in [1, \nt]$, and $x^\ell \in \cX^\ell$,
		\[
		\probof{X^{\ell} = x^\ell,\tau \ge \ell} = \min\{\mbig(x^\ell), 
		\msmall(x^\ell)\}.
		\]
	\end{lemma}

	\begin{proof}
		We prove this by induction in the backward direction. When $\ell = 
		\nt$, Step~12-14 in 
		\cref{alg:block_verify_greedy} accepts $X^\nt$ with probability 
		\[
			\probof{\tau = \nt \mid X^\nt = x^\nt} = \min\{1, \tpab_{\nt}\} = 
			\min \left \{1, 
			\frac{\mbig(x^\nt)}{\msmall(x^\nt)} \right\},
		\] 
		and hence
		\begin{align*}
			\probof{X^\nt = x^\nt,\tau \ge \nt}  & =  \probof{X^\nt = x^\nt}  
			\probof{\tau = \nt \mid X^\nt = x^\nt}
			= \min \left \{\msmall(x^\nt) , \mbig(x^\nt) \right\}.
		\end{align*}
		
		Suppose the equation holds for $\ell \ge \ell_0$, for $\ell = \ell_0 - 
		1$, we have
		
		\begin{align*}
			\probof{X^{\ell_0-1}  = x^{\ell_0-1}, \tau \ge \ell_0-1} 
			= \probof{X^{\ell_0-1} = x^{\ell_0 - 1}, \tau \ge \ell_0} + 
			\probof{X^{\ell_0-1} = x^{\ell_0-1}, \tau =  \ell_0-1}.
		\end{align*}
		
		Next we consider the two terms separately. For the first term, due to 
		the induction assumption, we have
		\[
		\probof{X^{\ell_0-1} = x^{\ell_0 - 1}, \tau \ge \ell_0} = \sum_{x \in 
		\cX} \min\{\cM_b(x^{\ell_0 -1}, x), \cM_s(x^{\ell_0 - 1}, x)\}
		\]
		For the second term, we have
		\begin{align}
			& \;\;\;\; \probof{X^{\ell_0-1} = x^{\ell_0-1}, \tau =  \ell_0-1} \nonumber \\
			& = \probof{X^{\ell_0-1} = x^{\ell_0-1}, \tau \le \ell_0-1} \cdot 
			\probof{X^{\ell_0-1} \text{ is accepted.}} 
			\label{eqn:backward_rejection} \\
			& = \probof{X^{\ell_0-1} = x^{\ell_0-1}, \tau \le \ell_0-1} \cdot 
			\probof{\eta_{\ell_0-1} \le h_{\ell_0-1}} \nonumber \\
			& = \Paren{\probof{X^{\ell_0-1} = x^{\ell_0-1}} - 
			\probof{X^{\ell_0-1} = x^{\ell_0-1}, \tau \ge \ell_0}} \cdot 
			\min\left\{1,  
			\frac{\premain(x^{\ell_0-1})}{\prej(x^{\ell_0-1})}\right\} \nonumber 
			\\
			& = \sum_{x\in \cX} \Paren{\cM_s(x^{\ell_0 - 1}, x) - 
			\min\{\cM_b(x^{\ell_0 -1}, x), \cM_s(x^{\ell_0 - 1}, x)\}} \cdot 
			\min\left\{1,  
			\frac{\premain(x^{\ell_0-1})}{\prej(x^{\ell_0-1})}\right\} 
			\label{eqn:induction}\\
			& = \prej(x^{\ell_0-1}) \cdot \min\left\{1,  
			\frac{\premain(x^{\ell_0-1})}{\prej(x^{\ell_0-1})}\right\}  
			\label{eqn:p_rej_def}\\
			& = \min \left\{\premain(x^{\ell_0-1}), \prej(x^{\ell_0-1}) \right\}. \nonumber
		\end{align}
		In the above derivation, \cref{eqn:backward_rejection} is due to that \cref{alg:block_verify_greedy} is outputing the longest accepted subblock. 
		\cref{eqn:induction} is due to the induction hypothesis. 
		\cref{eqn:p_rej_def} is due to the definition of $\prej$ in 
		\cref{eqn:prej_block}.
		
		Combining the two terms, we have
		\begin{align}
			& \;\;\; \,\, \probof{X^{\ell_0-1}  = x^{\ell_0-1}, \tau \ge \ell_0-1} 
			\nonumber \\
			& = \probof{X^{\ell_0-1} = x^{\ell_0 - 1}, \tau \ge \ell_0} + 
			\probof{X^{\ell_0-1} = x^{\ell_0-1}, \tau =  \ell_0-1}  \nonumber  
			\\
			& = \sum_{x \in \cX} \min\{\cM_b(x^{\ell_0 -1}, x), 
			\cM_s(x^{\ell_0 - 1}, x)\} + \min \left\{\premain(x^{\ell_0-1}), 
			\prej(x^{\ell_0-1}) \right\} \nonumber  \\
			& = \min \{ \sum_{x \in \cX} \min\{\cM_b(x^{\ell_0 -1}, x), 
			\cM_s(x^{\ell_0 - 1}, x)\}  +\premain(x^{\ell_0-1}),  \nonumber  \\
			& \qquad \qquad \qquad \sum_{x \in \cX} \min\{\cM_b(x^{\ell_0 
			-1}, x), \cM_s(x^{\ell_0 - 1}, x)\}  +\prej(x^{\ell_0-1}) \} 
			\nonumber  \\
			& = \min \left\{ \sum_{x \in \cX}\cM_b(x^{\ell_0 -1}, x),\sum_{x 
			\in \cX}\cM_s(x^{\ell_0 -1}, x) \right\} 
			\label{eqn:derive_premain_prej}\\
			& =  \min \left\{ \cM_b(x^{\ell_0 -1}),\cM_s(x^{\ell_0 -1}) \right\}, 
			\nonumber 
		\end{align}
		which is the desired quantity in the lemma. This concludes the proof. 
		Here \cref{eqn:derive_premain_prej} is due to the definition of 
		$\premain$ and $\prej$ in \cref{eqn:premain_block} and 
		\cref{eqn:prej_block}.
	\end{proof}
	
	Next we proceed to prove \cref{lem:greedy_block_identical}. Let $O^\nt = 
	(X^\tau, Y, Z^{\nt - \tau - 1})$. It would be enough to show that for all $i \le \nt$ 
	and $x^i \in \cX^i$,
	\[
	\probof{O^i = x^{i}} = \cM_b(x^{i}).
	\]
	
	We prove this via induction. Note that the corollary holds for $i = 0$, 
	which is the trivial case and both sides are equal to 1.
	Suppose the claim holds for $i \le i_0$. This means $\forall x^{i_0} \in 
	\cX^{i_0}$, we have
	\[
	\probof{O^{i_0} = x^{i_0}} = \cM_b(x^{i_0}).
	\]
	
	When $i = i_0+1$, by the algorithm, we have that either $\tau \ge i_0+1$, where $O_{i_0+1}$ is 
	an accepted token, or $\tau \le i_o$, where $O_{i_0+1}$ is sampled according to $\mnew(\cdot \mid O^{i_0})$ (or $\pres^{\rm greedy}(\cdot \mid O^{i_0})$, which is the same as $\mnew(\cdot \mid O^{i_0})$). By \cref{lem:block-accept}, we have
	\begin{align}
		& \;\;\;\; \probof{O^{i_0 + 1} = x^{i_0 + 1}} \nonumber \\
		& = \probof{O^{i_0 + 1} = x^{i_0 + 1}, \tau \ge i_0 +1} + 
		\probof{O^{i_0 + 1} = x^{i_0 + 1}, \tau \le i_0} \nonumber  \\
		& = \probof{X^{i_0 + 1} = x^{i_0 + 1}, \tau \ge i_0 +1}   +
		\probof{O^{i_0} = x^{i_0}, \tau \le i_0} \cdot \mnew(x_{i_0 + 1} \mid x^{i_0}) \nonumber  \\
		& = \probof{X^{i_0 + 1} = x^{i_0 + 1}, \tau \ge i_0 +1} \nonumber \\
		& \quad \quad  \quad \quad  \quad \quad  \quad \quad + 
		\Paren{\probof{O^{i_0} = x^{i_0}} - \probof{O^{i_0} = x^{i_0}, \tau \ge 
		i_0 +1}} \cdot \mnew(x_{i_0 + 1} \mid x^{i_0})) \nonumber \\
		& = \min\{ \cM_b(x^{i_0 + 1}), \cM_s(x^{i_0 + 1})\} \nonumber \\
		& \quad \quad  \quad \quad  \quad \quad  \quad \quad + 
		\Paren{\cM_b(x^{i_0}) - \sum_x \min\{ \cM_b(x^{i_0 }, x), 
		\cM_s(x^{i_0}, x)\} } \cdot \mnew(x_{i_0 + 1} \mid x^{i_0})) 
		\label{eqn:lemma_dist_match_induction}\\
		& = \min\{ \cM_b(x^{i_0 + 1}), \cM_s(x^{i_0 + 1})\}  +  \sum_x 
		\max\{\cM_b(x^{i_0}, x)- \cM_s(x^{i_0}, x), 0\} \cdot 
		\mnew(x_{i_0 + 1} \mid x^{i_0})) \nonumber\\
		& = \min\{ \cM_b(x^{i_0 + 1}), \cM_s(x^{i_0 + 1})\} + 
		\max\{\cM_b(x^{i_0}, x_{i_0+1})- \cM_s(x^{i_0}, x_{i_0+1}), 0\}
		\label{eqn:lemma_dist_match_pres_def}\\
		& = \cM_b(x^{i_0 + 1}). \nonumber
	\end{align}
	Here \cref{eqn:lemma_dist_match_induction} follows by the induction 
	hypothesis, and \cref{eqn:lemma_dist_match_pres_def} follows by the 
	definition of $\mnew$ in \cref{eqn:model_new}. By 
	induction, this concludes the proof.

\subsection{Proof of 
\cref{thm:greedy_one_iter}}\label{proof:greedy_optimal}

To prove \cref{thm:greedy_one_iter}, we first observe the following: For \cref{alg:block_verify_greedy}, let $\tau$ be the number of accepted tokens, due to \cref{lem:block-accept}, we have 
\begin{align*}
    \EE_{X^\nt \sim \msmall^\nt}[\tau] & = \sum_{\ell = 1}^\nt \probof{\tau \ge \ell} = \sum_{\ell = 1}^\nt \sum_{x^\ell \in \cX^\ell} \probof{X^\ell = x^\ell, \tau \ge \ell} \\
    & = \sum_{\ell = 1}^\nt \sum_{x^\ell \in \cX^\ell} \min\{\cM_s(x^\ell), \cM_b(x^\ell)\}.
\end{align*}

Next we show that the above expected acceptance length is optimal for a family of draft verification algorithms that performs a coupling between sample blocks from the draft and target distributions. For all $x^\nt, y^\nt\in \cX^\nt$, let the \emph{maximum common prefix length} be defined as $$
    \acclength(x^\nt, y^\nt) \eqdef \max_{\ell \le \nt} \{\forall i \le \ell, x_i = y_i\}.$$ 
Formally, let $\pi$ be a joint distribution over $\cX^\nt \times \cX^\nt$, then \cref{alg:block_verify_greedy} solves the following optimization problem.
\begin{align}
    \max_{\pi}  \EE_{X^\nt, y^\nt \sim \pi} \left[\acclength(X^\nt, y^\nt) \right], \label{eqn:sequence-optimization}
    \end{align}
 subject to constraints
    \begin{align}
     \sum_{y^\nt} \pi(x^\nt, y^\nt) & =  \cM_s(x^\nt), \;\; \forall x^\nt \in \cX^\nt,  \label{eqn:left_margin}
     \\
      \sum_{x^\nt} \pi(x^\nt,y^\nt) & =  \cM_b(y^\nt), \;\; \forall y^\nt \in \cX^\nt. \label{eqn:right_margin}
\vspace{-.1in}
\end{align}

    In this above formulation, the marginal distributions satisfy $X^\nt \sim \msmall^\nt$ and $Y^\nt \sim \mbig^\nt$. And the maximum common prefix length refers to the number of accepted tokens in one iteration of speculative decoding.
Note that the optimization problem can be viewed as an optimal transport problem~\citep{villani2009optimal} between distributions $\cM_s^\nt(\cdot)$ and $\cM_b^\nt(\cdot)$ with the cost function being $(\nt - \acclength(X^\nt, Y^\nt)$. The next lemma establishes the optimality of \cref{alg:block_verify_greedy} in solving this problem.

\begin{lemma} \label{lem:upper_bound}
The solution to \cref{eqn:sequence-optimization} is upper bounded by
\begin{align*}
     \sum_{\tau = 1}^\nt \sum_{x^\tau  \in \cX^\tau} \min\{\cM_s(x^\ell), \cM_b(x^\ell)\} 
\end{align*}
\end{lemma}

\begin{proofof}{\cref{lem:upper_bound}}
For all $\pi$ that satisfies \cref{eqn:left_margin,eqn:right_margin} and $X^\nt, Y^\nt \sim \pi$, we have
\begin{align*}
 \EE_{X^\nt, Y^\nt}[\acclength(X^\nt, Y^\nt)] %
     & \stackrel{(a)}{=} \sum_{\ell \le \nt} \text{Pr}_{X^\nt, Y^\nt} \left(\acclength(X^\nt, Y^\nt) \ge \ell\right)\\
     & = \sum_{\ell \le \nt} \sum_{x^\ell} \text{Pr}_{X^\nt, Y^\nt } \left(X^{\ell} = Y^{\ell} =x^\ell, \acclength(X^\nt, Y^\nt) \ge \ell\right) \\
     & \le \sum_{\ell \le \nt} \sum_{x^\ell} \text{Pr}_{X^\nt, Y^\nt } \left(X^{\ell} = Y^{\ell} =x^\ell \right) \\
     & \stackrel{(b)}{\le} \sum_{\ell \le \nt} \sum_{x^\ell} \min\left\{  \text{Pr} \left(X^{\ell} = x^\ell\right), \text{Pr} \left(Y^{\ell} =x^\ell \right) \right\}\\
     & = \sum_{\ell \le \nt} \sum_{x^\ell} \min\{ \cM_s^\ell(x^\ell), \cM_b^\ell(x^\ell)\}.
\end{align*}
Here $(a)$ follows from the fact that for a positive integer random variable $\EE[X] = \sum_{i} \Pr(X \geq i)$; $(b)$ follows from the fact that the joint probability is upper bounded by the minimum of the marginals.
\end{proofof}

Then \cref{thm:greedy_one_iter} holds by noticing that block verification \cref{alg:block_verify} is also an instance of the coupling by setting $X^\nt$ to be the draft tokens and $Y^\nt = (X^\tau, Y', Z^{\nt - \tau - 1})$ where $(X^\tau, Y')$ are the outputs from \cref{alg:block_verify}  and $Z^{\nt - \tau - 1} \sim \mbig(\cdot \mid \context, X^\tau, Y)$ (\cref{lem:verify_goal}).

%% file: sections/app_experiment.tex
\section{Additional Experimental Results} \label{sec:experiments_app}

\subsection{Comparison to speculative decoding with multiple drafts.}\label{app:multi-draft}

\znew{Recent works \citep{sun2023spectr, miao2023specinfer} have extended speculative decoding to the case with multiple draft blocks to improve block efficiency. However, these methods also increase the required computation from the large model to verify the drafts. In high-throughput LLM serving systems, query batching \citep{kwon2023batching} is a common technique where multiple prefixes are decoded at the same time. In these cases, the inference will be less memory bound and there will not be enough extra parallel compute to evaluate the increased number of drafts without decreasing latency.}

\znew{We empirically compare block verification and SpecTr \citep{sun2023spectr}, SpecInfer \citep{miao2023specinfer} with query batching. We set the batch size $B=8$ and use PALM-2-XXS as the draft model. In \cref{tab:multi_draft}, we list the wall clock speedup and block efficiency for $\nt = 8$. The number of draft blocks for SpecTr and SpecInfer are taken to be 2, which is the one that achieves the lowest latency over $\{2, 4, 8\}$ when $B = 8, \nt = 8$.}

\begin{table*}[h]
\caption{$\nt = 8$, $B = 8$. Speedup comparison between token verification (\textsc{TokenV}) and block verification (\textsc{BlockV}) with PALM-2-XXS as the draft model on various datasets and tasks.}
 \setlength{\tabcolsep}{2pt}
\begin{center}
\begin{tabular}{  c c c c c c c c c}
\toprule
\multirow{2}{*}{Dataset} &  \multicolumn{4}{c}{Wall clock time  over baseline} & \multicolumn{4}{c}{Block efficiency} \\
\cmidrule(lr){2-5} \cmidrule(lr){6-9} 
& \textsc{TokenV} & \textsc{BlockV}   & SpecTr & SpecInfer & \textsc{TokenV} & \textsc{BlockV}  & SpecTr & SpecInfer\\ 
 \midrule
GPT Prompt & 1.300 & \textbf{1.381} & 1.290 & 1.263  & 3.394  & 3.715 & \textbf{3.898} & 3.833\\
WebQA & 1.302 & \textbf{1.368} & 1.279  & 1.274& 3.451 & 3.7 & \textbf{3.933} & 3.894\\
ShareGPT & 1.267 & \textbf{1.333} & 1.244 & 1.236 & 3.366 & 3.63 & \textbf{3.824} & 3.78\\
GSM8K & 1.353 & \textbf{1.445} & 1.344 & 1.319 & 3.856 & 4.179 & \textbf{4.356} & 4.277\\
XSum& 1.328 & \textbf{1.403} & 1.300 & 1.285 & 3.487 & 3.768 & \textbf{3.949} & 3.897 \\
PIQA & 1.305 & \textbf{1.377} & 1.270 & 1.280 & 3.401 & 3.685 & \textbf{3.846} & 3.82 \\
LM1B & 1.274 & \textbf{1.344} & 1.253 & 1.245 & 3.218 & 3.494 & \textbf{3.669} & 3.629\\
WMT-DeEn & 1.222 & \textbf{1.293} & 1.204 & 1.194 & 3.165 & 3.422 & \textbf{3.603} & 3.56 \\
  \bottomrule
\end{tabular}
\end{center}
\label{tab:multi_draft}
\end{table*}

\znew{We observe that while SpecTr and SpecInfer can achieve higher block efficiencies, due to the increased computation to evaluate more candidates, our method achieves better speedup than SpecTr and SpecInfer, demonstrating the advantage of our method in the common practical setting with query batching.}

\subsection{Detailed results with other parameter settings}
In this section, we present experimental results for the same set of experiments described in \cref{sec:experiments} with different block lengths ($\nt = 4, 6, 8$) and different drafters (PALM-2-XXS and PALM-2-XXXS):
\begin{itemize}
    \item \cref{tab:exp_comparison_l4}. Drafter: PALM-2-XXS, $\nt=4$.
    \item \cref{tab:exp_comparison_l6}. Drafter: PALM-2-XXS, $\nt=6$.
    \item \cref{tab:exp_comparison_l4_xxxs}. Drafter: PALM-2-XXXS, $\nt=4$.
    \item \cref{tab:exp_comparison_l6_xxxs}. Drafter: PALM-2-XXXS, $\nt=6$.
    \item \cref{tab:exp_comparison_l8_xxxs}. Drafter: PALM-2-XXXS, $\nt=8$.
\end{itemize}

\begin{table*}[h!]
\vspace{-10pt}
\caption{Speedup comparison between token verification (\textsc{TokenV}) and block verification (\textsc{BlockV}) with $\nt = 4$ and PALM-2-XXS being the draft model. Each statistic is computed using $1000$ test prompts from different datasets on various tasks (each run is an average with 3 different random seeds). Numbers after $\pm$ represent standard deviation.}
 \setlength{\tabcolsep}{2pt}
\begin{center}
\begin{tabular}{  c c c c c c c}
\toprule
\multirow{2}{*}{Dataset} &  \multicolumn{3}{c}{Block efficiency} & \multicolumn{3}{c}{Wall clock time speedup over baseline} \\
\cmidrule(lr){2-4} \cmidrule(lr){5-7} 
& \textsc{TokenV} & \textsc{BlockV}   & Improve. $\uparrow$ \% & \textsc{TokenV} & \textsc{BlockV}  & Improve. $\uparrow$ \% \\ 
 \midrule
LM1B & $2.78\pm 0.01$ & $2.88\pm 0.01$ & $3.48\pm 0.24$ & $2.36\pm 0.00$ & $2.42\pm 0.01$ & $2.51\pm 0.22$ \\
GPT Prompt & $2.88\pm 0.01$ & $3.00\pm 0.00$ & $4.33\pm 0.25$ & $2.43\pm 0.01$ & $2.51\pm 0.00$ & $3.43\pm 0.24$ \\
WebQA & $2.91\pm 0.01$ & $2.99\pm 0.01$ & $2.83\pm 0.65$ & $2.45\pm 0.01$ & $2.50\pm 0.01$ & $1.94\pm 0.61$ \\
PIQA & $2.89\pm 0.00$ & $2.99\pm 0.01$ & $3.48\pm 0.21$ & $2.44\pm 0.00$ & $2.50\pm 0.01$ & $2.66\pm 0.20$ \\
ShareGPT & $2.85\pm 0.01$ & $2.95\pm 0.00$ & $3.48\pm 0.19$ & $2.41\pm 0.01$ & $2.47\pm 0.00$ & $2.63\pm 0.17$ \\
XSum & $2.94\pm 0.01$ & $3.03\pm 0.01$ & $3.24\pm 0.51$ & $2.48\pm 0.01$ & $2.54\pm 0.01$ & $2.35\pm 0.48$ \\
GSM8K & $3.12\pm 0.01$ & $3.21\pm 0.02$ & $3.06\pm 0.95$ & $2.62\pm 0.01$ & $2.68\pm 0.02$ & $2.19\pm 0.89$ \\
WMT-DeEn & $2.75\pm 0.01$ & $2.83\pm 0.01$ & $2.99\pm 0.09$ & $2.33\pm 0.01$ & $2.38\pm 0.01$ & $2.18\pm 0.09$ \\
\midrule
Average & $2.89$&	$2.99$ & $3.36$ &	$2.44$ &	$2.50$ &$2.49$ \\
  \bottomrule
\end{tabular}
\end{center}
\vspace{-10pt}
\label{tab:exp_comparison_l4}
\end{table*}

\begin{table*}[h!]
\caption{Speedup comparison between token verification (\textsc{TokenV}) and block verification (\textsc{BlockV}) with $\nt = 6$ and PALM-2-XXS being the draft model. Each statistic is computed using $1000$ test prompts from different datasets on various tasks (each run is an average with 3 different random seeds). Numbers after $\pm$ represent standard deviation.}
 \setlength{\tabcolsep}{2pt}
\begin{center}
\begin{tabular}{  c c c c c c c}
\toprule
\multirow{2}{*}{Dataset} &  \multicolumn{3}{c}{Block efficiency} & \multicolumn{3}{c}{Wall clock time speedup over baseline} \\
\cmidrule(lr){2-4} \cmidrule(lr){5-7} 
& \textsc{TokenV} & \textsc{BlockV}   & Improve. $\uparrow$ \% & \textsc{TokenV} & \textsc{BlockV}  & Improve. $\uparrow$ \% \\ 
 \midrule
LM1B & $3.08\pm 0.01$ & $3.27\pm 0.01$ & $6.42\pm 0.07$ & $2.32\pm 0.01$ & $2.43\pm 0.01$ & $5.00\pm 0.06$ \\
GPT Prompt & $3.22\pm 0.01$ & $3.44\pm 0.02$ & $6.55\pm 0.83$ & $2.42\pm 0.00$ & $2.54\pm 0.02$ & $5.06\pm 0.77$ \\
WebQA & $3.26\pm 0.01$ & $3.44\pm 0.01$ & $5.60\pm 0.22$ & $2.45\pm 0.01$ & $2.55\pm 0.01$ & $4.24\pm 0.21$ \\
PIQA & $3.22\pm 0.02$ & $3.43\pm 0.02$ & $6.36\pm 0.78$ & $2.42\pm 0.01$ & $2.54\pm 0.01$ & $4.92\pm 0.72$ \\
ShareGPT & $3.18\pm 0.02$ & $3.37\pm 0.01$ & $6.13\pm 0.53$ & $2.39\pm 0.02$ & $2.50\pm 0.01$ & $4.74\pm 0.49$ \\
XSum & $3.29\pm 0.01$ & $3.48\pm 0.01$ & $5.91\pm 0.82$ & $2.47\pm 0.01$ & $2.58\pm 0.01$ & $4.47\pm 0.77$ \\
GSM8K & $3.56\pm 0.01$ & $3.80\pm 0.03$ & $6.86\pm 0.60$ & $2.66\pm 0.01$ & $2.80\pm 0.02$ & $5.38\pm 0.56$ \\
WMT-DeEn & $3.04\pm 0.01$ & $3.19\pm 0.01$ & $4.92\pm 0.29$ & $2.29\pm 0.01$ & $2.37\pm 0.01$ & $3.57\pm 0.27$ \\
\midrule
Average & $3.23$ & $ 	3.43$ & $	6.10$ & $	2.43	$ & $ 2.54	$ & $4.67$ \\
  \bottomrule
\end{tabular}
\end{center}
\label{tab:exp_comparison_l6}
\end{table*}

\begin{table*}[h!]
\caption{Speedup comparison between token verification (\textsc{TokenV}) and block verification (\textsc{BlockV}) with $\nt = 4$ and PALM-2-XXXS being the draft model. Each statistic is computed using $1000$ test prompts from different datasets on various tasks (each run is an average with 3 different random seeds). Numbers after $\pm$ represent standard deviation.}
 \setlength{\tabcolsep}{2pt}
\begin{center}
\begin{tabular}{  c c c c c c c}
\toprule
\multirow{2}{*}{Dataset} &  \multicolumn{3}{c}{Block efficiency} & \multicolumn{3}{c}{Wall clock time speedup over baseline} \\
\cmidrule(lr){2-4} \cmidrule(lr){5-7} 
& \textsc{TokenV} & \textsc{BlockV}   & Improve. $\uparrow$ \% & \textsc{TokenV} & \textsc{BlockV}  & Improve. $\uparrow$ \% \\ 
 \midrule
LM1B & $2.24\pm 0.00$ & $2.33\pm 0.01$ & $4.23\pm 0.44$ & $2.25\pm 0.00$ & $2.34\pm 0.01$ & $3.89\pm 0.41$ \\
GPT Prompt & $2.41\pm 0.02$ & $2.48\pm 0.01$ & $2.96\pm 1.00$ & $2.42\pm 0.02$ & $2.48\pm 0.01$ & $2.72\pm 0.94$ \\
WebQA & $2.38\pm 0.01$ & $2.45\pm 0.01$ & $2.87\pm 0.13$ & $2.39\pm 0.01$ & $2.45\pm 0.01$ & $2.63\pm 0.12$ \\
PIQA & $2.36\pm 0.01$ & $2.43\pm 0.01$ & $3.22\pm 0.37$ & $2.37\pm 0.01$ & $2.44\pm 0.01$ & $2.97\pm 0.35$ \\
ShareGPT & $2.34\pm 0.00$ & $2.42\pm 0.01$ & $3.49\pm 0.12$ & $2.35\pm 0.00$ & $2.42\pm 0.01$ & $3.16\pm 0.12$ \\
XSum & $2.38\pm 0.01$ & $2.45\pm 0.01$ & $2.91\pm 0.63$ & $2.39\pm 0.01$ & $2.45\pm 0.01$ & $2.68\pm 0.60$ \\
GSM8K & $2.51\pm 0.01$ & $2.58\pm 0.02$ & $2.99\pm 0.47$ & $2.51\pm 0.01$ & $2.58\pm 0.02$ & $2.74\pm 0.44$ \\
WMT-DeEn & $2.22\pm 0.00$ & $2.28\pm 0.00$ & $2.59\pm 0.09$ & $2.24\pm 0.00$ & $2.29\pm 0.00$ & $2.37\pm 0.08$ \\
\midrule
Average & $2.35$ & $	2.43$ & $	3.16$ & $	2.36 $ & $	2.43$ & $	2.89$ \\
  \bottomrule
\end{tabular}
\end{center}
\label{tab:exp_comparison_l4_xxxs}
\end{table*}

\begin{table*}[h!]
\caption{Speedup comparison between token verification (\textsc{TokenV}) and block verification (\textsc{BlockV}) with $\nt = 6$ and PALM-2-XXXS being the draft model. Each statistic is computed using $1000$ test prompts from different datasets on various tasks (each run is an average with 3 different random seeds). Numbers after $\pm$ represent standard deviation.}
 \setlength{\tabcolsep}{2pt}
\begin{center}
\begin{tabular}{  c c c c c c c}
\toprule
\multirow{2}{*}{Dataset} &  \multicolumn{3}{c}{Block efficiency} & \multicolumn{3}{c}{Wall clock time speedup over baseline} \\
\cmidrule(lr){2-4} \cmidrule(lr){5-7} 
& \textsc{TokenV} & \textsc{BlockV}   & Improve. $\uparrow$ \% & \textsc{TokenV} & \textsc{BlockV}  & Improve. $\uparrow$ \% \\ 
 \midrule
LM1B & $2.36\pm 0.01$ & $2.48\pm 0.00$ & $4.93\pm 0.46$ & $2.27\pm 0.01$ & $2.37\pm 0.00$ & $4.55\pm 0.43$ \\
GPT Prompt & $2.58\pm 0.04$ & $2.72\pm 0.02$ & $5.57\pm 1.29$ & $2.46\pm 0.03$ & $2.59\pm 0.01$ & $5.10\pm 1.22$ \\
WebQA & $2.54\pm 0.00$ & $2.68\pm 0.02$ & $5.46\pm 0.50$ & $2.43\pm 0.00$ & $2.55\pm 0.01$ & $5.02\pm 0.47$ \\
PIQA & $2.50\pm 0.00$ & $2.62\pm 0.01$ & $5.06\pm 0.39$ & $2.39\pm 0.00$ & $2.50\pm 0.01$ & $4.66\pm 0.37$ \\
ShareGPT & $2.47\pm 0.01$ & $2.60\pm 0.01$ & $5.10\pm 0.49$ & $2.37\pm 0.01$ & $2.48\pm 0.01$ & $4.69\pm 0.46$ \\
XSum & $2.54\pm 0.01$ & $2.67\pm 0.01$ & $4.83\pm 0.47$ & $2.43\pm 0.01$ & $2.54\pm 0.01$ & $4.45\pm 0.44$ \\
GSM8K & $2.71\pm 0.03$ & $2.83\pm 0.00$ & $4.27\pm 0.89$ & $2.58\pm 0.02$ & $2.69\pm 0.00$ & $3.92\pm 0.84$ \\
WMT-DeEn & $2.31\pm 0.01$ & $2.43\pm 0.02$ & $5.38\pm 0.57$ & $2.21\pm 0.00$ & $2.32\pm 0.01$ & $4.99\pm 0.54$ \\
\midrule
Average & $2.50	$ & $2.63$ & $	5.07$ & $	2.39$ & $	2.50$ & $	4.67$ \\
  \bottomrule
\end{tabular}
\end{center}
\label{tab:exp_comparison_l6_xxxs}
\end{table*}

\begin{table*}[h!]
\caption{Speedup comparison between token verification (\textsc{TokenV}) and block verification (\textsc{BlockV}) with $\nt = 8$ and PALM-2-XXXS being the draft model. Each statistic is computed using $1000$ test prompts from different datasets on various tasks (each run is an average with 3 different random seeds). Numbers after $\pm$ represent standard deviation.}
 \setlength{\tabcolsep}{2pt}
\begin{center}
\begin{tabular}{  c c c c c c c}
\toprule
\multirow{2}{*}{Dataset} &  \multicolumn{3}{c}{Block efficiency} & \multicolumn{3}{c}{Wall clock time speedup over baseline} \\
\cmidrule(lr){2-4} \cmidrule(lr){5-7} 
& \textsc{TokenV} & \textsc{BlockV}   & Improve. $\uparrow$ \% & \textsc{TokenV} & \textsc{BlockV}  & Improve. $\uparrow$ \% \\ 
 \midrule
LM1B & $2.40\pm 0.01$ & $2.55\pm 0.01$ & $6.19\pm 0.43$ & $2.13\pm 0.01$ & $2.25\pm 0.01$ & $5.28\pm 0.40$ \\
GPT Prompt & $2.66\pm 0.01$ & $2.82\pm 0.02$ & $6.28\pm 1.01$ & $2.35\pm 0.01$ & $2.47\pm 0.02$ & $5.37\pm 0.95$ \\
WebQA & $2.61\pm 0.01$ & $2.78\pm 0.00$ & $6.27\pm 0.49$ & $2.31\pm 0.01$ & $2.43\pm 0.00$ & $5.39\pm 0.46$ \\
PIQA & $2.57\pm 0.01$ & $2.76\pm 0.01$ & $7.48\pm 0.51$ & $2.27\pm 0.01$ & $2.42\pm 0.01$ & $6.51\pm 0.47$ \\
ShareGPT & $2.54\pm 0.01$ & $2.71\pm 0.01$ & $6.63\pm 0.72$ & $2.25\pm 0.01$ & $2.38\pm 0.01$ & $5.68\pm 0.68$ \\
XSum & $2.60\pm 0.01$ & $2.77\pm 0.00$ & $6.46\pm 0.49$ & $2.30\pm 0.01$ & $2.43\pm 0.00$ & $5.53\pm 0.46$ \\
GSM8K & $2.82\pm 0.02$ & $2.98\pm 0.03$ & $5.48\pm 1.18$ & $2.49\pm 0.01$ & $2.60\pm 0.03$ & $4.62\pm 1.11$ \\
WMT-DeEn & $2.37\pm 0.00$ & $2.49\pm 0.01$ & $5.33\pm 0.46$ & $2.10\pm 0.00$ & $2.20\pm 0.01$ & $4.53\pm 0.43$ \\
\midrule
Average & $2.57	$ & $2.73$ & $	6.27$ & $	2.28$ & $	2.40$ & $	5.36$ \\
  \bottomrule
\end{tabular}
\end{center}
\label{tab:exp_comparison_l8_xxxs}
\end{table*}